\def\year{2020}\relax
\documentclass[letterpaper]{article} 
\usepackage{aaai20}  
\usepackage{times}  
\usepackage{helvet} 
\usepackage{courier}  
\usepackage[hyphens]{url}  
\usepackage{graphicx} 
\usepackage{graphicx}  
\usepackage{algorithm}
\usepackage{algorithmic}

\usepackage{booktabs}
\usepackage{amsfonts}
\usepackage{amsmath}
\newtheorem{theorem}{Theorem}
\newtheorem{lemma}{Lemma}

\newtheorem{assumption}{Assumption}
\newtheorem{remark}{Remark}
\newtheorem{proof}{Proof}
\usepackage{multirow}
\usepackage{subcaption}

\title{Quadruply Stochastic Gradient Method for Large Scale Nonlinear Semi-Supervised Ordinal Regression AUC Optimization }

\author{
	Wanli Shi\textsuperscript{\rm 1},\	
	Bin Gu\textsuperscript{\rm 1,2}\thanks{Corresponding Author},
	Xiang Li\textsuperscript{\rm 3},\
	Heng Huang\textsuperscript{\rm 4,2}\
	\\
	\textsuperscript{\rm 1} School of Computer \& Software, Nanjing University of Information Science \& Technology, P.R.China\\
	\textsuperscript{\rm 2}JD Finance America Corporation\\
	\textsuperscript{\rm 3} Computer Science Department, University of Western Ontario, Canada\\ 	
	\textsuperscript{\rm 4}Department of Electrical \& Computer Engineering, University of Pittsburgh, USA\\	
	wanlishi@nuist.edu.cn,
	jsgubin@gmail.com,
	lxiang2@uwo.ca,
	heng.huang@pitt.edu
}

\begin{document}

\maketitle

\begin{abstract}
	Semi-supervised ordinal regression (S$^2$OR) problems are ubiquitous in real-world applications, where only a few ordered instances are labeled and massive instances remain unlabeled. Recent researches have shown that directly optimizing concordance index or AUC can impose a better ranking on the data than optimizing the traditional error rate in ordinal regression (OR) problems. In this paper, we propose an unbiased objective function for S$^2$OR AUC optimization based on ordinal binary decomposition approach. Besides, to handle the large-scale kernelized learning problems, we propose a scalable algorithm called QS$^3$ORAO using the doubly stochastic gradients (DSG) framework for functional optimization. Theoretically, we prove that our method can converge to the optimal solution at the rate of $O(1/t)$, where $t$ is the number of iterations for stochastic data sampling. Extensive experimental results on various benchmark and real-world datasets also demonstrate that our method is efficient and effective while retaining similar generalization performance. 
\end{abstract}

\section{Introduction}
Supervised ordinal regression (OR) problems have made great process in the past few decades, such as \cite{chu2007support,fathony2017adversarial,niu2016ordinal,Gu2015IncrementalSV,gu2018regularization}. However, in various practical fields, such as facial beauty assessment \cite{yan2014cost}, credit rating \cite{kim2012corporate}, social sciences \cite{fullerton2012proportional} or more, collecting a large amount of ordinal labeled instances is time-consuming, while unlabeled data are available in abundance. Often, the finite ordinal data are insufficient to learn a good ordinal regression model. To improve the performance of the classifiers, one needs to incorporate unlabeled instances into the training process. So far, semi-supervised ordinal regression (S$^2$OR) problems have attracted great attention in machine learning communities, such as \cite{srijith2013semi,seah2012transductive}.

To evaluate the performance of an OR model, many metrics could be used, \textit{e.g.}, the mean absolute error, the mean squared error. However, Waegeman, De Baets and Boullart, \shortcite{waegeman2008roc} have shown that OR models which minimize these errors do not necessarily impose a good ranking on  the data. To handle this problem, many researchers start to use AUC or concordance index in solving OR problems since AUC is defined on an ordinal scale, such as \cite{waegeman2008roc,waegeman2010survey,furnkranz2009binary,uematsu2014statistical}.  We summarized several representative OR algorithms in Table \ref{tab:algorithms}. 

However, existing AUC optimization methods focus on supervised OR problems, and none of them can be applied to semi-supervised learning problems. The main challenge is how to incorporate unlabeled instances into the AUC optimization process. For the semi-supervised learning research field in general, many existing methods, such as \cite{seah2012transductive,fujino2016semi}, have leveraged the cluster assumptions, which states that similar instances tend to share the same label, to solve this problem. However, the clustering assumption is rather restrictive and may mislead a model towards a biased solution. Nevertheless, recent works \cite{sakai2018semi,xie2018semi} have shown that the clustering assumption is actually unnecessary at least for binary classification problems. In the same vein, we propose an objective function of S$^2$OR AUC optimization based on ordinal binary decomposition without using the clustering assumption. Specifically, for a $k$ classes OR problem, we use $k-1$ hyperlanes to decompose the orginal problem into $k-1$ binary semi-supervised AUC optimization problems, where the AUC risk can be viewed as a linear combination of AUC risk between labeled instances and AUC risk between labeled and unlabeled instances. Then, the overall AUC risk in S$^2$OR is equivalent to the mean of AUC for $k-1$ subproblems.

\begin{table*}[]
	\centering
	\caption{Several representative OR algorithms. ($D$ denotes the number of random features, $k$ denotes the number of classes, $n$ denotes the number of training samples, and $t$ denotes number of iterations.)}
	\label{tab:algorithms}
	\begin{tabular}{lllccc}
		\toprule
		Learning setting                 & Algorithm     & Reference & AUC & Computational complexity & Space complexity \\
		\hline
		\multirow{4}{*}{Supervised}      & ALOR        &    Fathony et al, \shortcite{fathony2017adversarial}      & No             & $O(n^3)$            & $O(n^2)$    \\
		&SVOREX       & Chu et al, \shortcite{chu2007support}&No &$O(n^3)$&$O(n^2)$ \\
		& VUS           &      Waegeman et al, \shortcite{waegeman2008roc}     & Yes            & $O( n^3)$         & $O(n^2)$  \\
		& MultiRank     &     Uematsu and Lee, \shortcite{uematsu2014statistical}     & Yes            & $O(n^3)$          & $O(n^2)$  \\
		\hline
		\multirow{5}{*}{Semi-supervised} & TOR           &    Seah et al, \shortcite{seah2012transductive}      & No             & $O(n^3)$            & $O(n^2)$    \\
		& SSORERM		&\cite{DBLP:journals/corr/abs-1901-11351} & No &$O(n^3)$&$O(n^2)$\\
		& SSGPOR        &    Srijith et al, \shortcite{srijith2013semi}       & No             & $>O(n^3)$                      & $O(n^2)$    \\
		& ManifoldOR    &    Liu et al, \shortcite{liu2011semi}       & No             & $O(n^3)$                      & $O(n^2)$    \\
		& QS$^3$ORAO &     Ours      & Yes            & $O(Dt^2)$                & $O(Dn)$     \\
		\bottomrule     
	\end{tabular}
\end{table*}

Nonlinear data structures widely exist in many real-world problems, and kernel method is a typical way to solve such problems. However, kernel-based methods are hardly scalable. Specifically, the kernel matrix needs $O(n^2d)$ operations to be calculated and $O(n^2)$ to be stored, where $n$ denotes the number of training data and $d$ denotes the dimensionality of the data. Besides, the bottlenecks of the computational complexities become more severe in solving pairwise learning problems such as AUC optimization. In addition, as required by AUC computation, the OR learning problem needs to be decomposed into several binary classification sub-problems, which further increases the problem size and computational complexity. Thus, the new challenge is how to scale up kernel-based S$^2$OR AUC optimization.

Scaling up kernel method has attracted greate attend in machining \cite{gu2018asynchronous,gu2019asynchronous,Shi2019QuadruplySG,dai2014scalable,gu2018accelerated}.
Recently, Dai et al, \shortcite{dai2014scalable} proposed  doubly stochastic gradient (DSG) method to scale up kernel-based algorithms. Specifically, in each iteration, DSG randomly samples a data instance and its random features to compute the doubly stochastic functional gradient, and then the model function can be updated by using this gradient. However, the original DSG cannot be applied to solve the kernel-based S$^2$OR AUC optimization. On the one hand, optimizing AUC is a pairwise problem which is much more complicated than the pointwise problem considered in standard DSG framework. On the other hand, S$^2$OR optimization problems need to handle two different types of data,  \textit{i.e.,} unlabeled dataset and the datasets of class $i$, while standard DSG focuses on minimizing the empirical risk on a single dataset with all data instances labeled.

To address these challenging problems, inspride by \cite{Shi2019QuadruplySG,gu2019scalable}, we introduce multiple sources of randomness. Specifically, we randomly sample a positive instance, a negative instance, an unlabeled instance, and their random features in each subproblem to calculate the approximated stochastic gradients of our objective function in each iteration. Then the ranking function can be iteratively updated. Since we randomly sample instances from four data sources in each subproblem, we denote our method as quadruply stochastic gradient S$^2$OR AUC optimization method (QS$^3$ORAO). Theoretically, we prove that our proposed QS$^3$ORAO can converge to the optimal solution at the rate of $O(1/t)$.  Extensive experimental results on benchmark datasets and real-world datasets also demonstrate that our method is efficient and effective while retaining similar generalization performance. 


\noindent \textbf{Contributions.} The  main contributions of this paper are summarized as follows.
\begin{enumerate}
	
	\item We propose an objective function for solving S$^2$OR AUC optimization problems in an unbiased manner. To the best of our knowledge, this is the first objective formulation incorporating the unlabeled data into the AUC optimization process in OR problems.
	
	\item To optimize the objective function under the kernel learning setting, we propose an efficient and scalable S$^2$OR AUC optimization algorithm, QS$^3$ORAO, based on DSG framework.
	
	\item We provide the convergence analysis of QS$^3$ORAO, which indicates that an ideal $O(1/t)$ convergence rate is possible under certain mild assumptions.
	
\end{enumerate}

\section{Related Works}
\subsection{Semi-Supervised Ordinal Regression}

In real-world applications, labeled instances are often costly to calibrate or difficult to obtain. This has led to a lot of efforts to study how to make full use of unlabeled data to improve the accuracy of classification, such as \cite{ijcai2019-590,han2018co,yu2019tackle}. Many existing methods incorporate unlabeled instances into learning propose by using various restrictive assumptions. For example, Seah, Tsang and Ong, \shortcite{seah2012transductive} proposed TOR based on cluster assumption, where the instances share the same label if there are close to each other. Liu et al, \shortcite{liu2011semi} proposed a semi-supervised OR method, ManifoldOR, based on the assumption that the input data are distributed into a lower-dimensional manifold \cite{belkin2006manifold}. Besides, Srijith et al, \shortcite{srijith2013semi} proposed SSGPOR based on the low density separation assumption \cite{chapelle2009semi}. We summarized these semi-supervised OR algorithms in Table \ref{tab:algorithms}. Note, in our semi-supervised OR AUC method, we do not need these restrictive assumptions.

\subsection{Kernel Approximation}
Kernel approximation is a common method to scale up kernel-based algorithms, which can be decomposed into two categories. One is data-dependent methods, such as greedy basis selection techniques \cite{smola2000sparse}, incomplete Cholesky decomposition \cite{fine2001efficient}, Nystr{\"o}m method \cite{drineas2005nystrom}. In order to achieve a low generalization performance, these methods usually need a large amount of training instances to compute a low-rank approximation of the kernel matrix, which may have high memory requriement. Another one is data-independent methods, which directly approximates the kernel function unbiasedly with some basis functions, such as random Fourier feature (RFF) \cite{rahimi2008random}. However, RFF method needs to save large amounts of random features. Instead of saving all the random features, Dai et al, \shortcite{dai2014scalable} proposed DSG algorithm to use \textit{pseudo-random number generators} to generate the random features on-the-fly, which  has been widely used \cite{Shi2019QuadruplySG,Geng2019ScalableSS,Li2017TriplySG}. Our method can be viewed as an extension of \cite{Shi2019QuadruplySG}. However, OR is much more complicated than binary classification, since OR involves $k$ classes with ordering constraint, while \cite{Shi2019QuadruplySG} only studies binary classification. How the $k$ ordered classes could be learnt under the DSG framework is a novel and challenging problem. Theoretically, whether and to what extent the convergence property remains true is also a non-trivial problem.

\section{Preliminaries}

In this section, we first give a brief review of the AUC optimization framework in supervised ordinal regression settings, and then we propose our objective function in S$^2$OR AUC optimization problems. Finally, we give a brief review of random Fourier features.
\subsection{Supervised Ordinal Regression AUC Optimization}
Let $x \in \mathbb{R}^d$ be a $d$-dimensional data instance and $y=\{1,\cdots,k\}$ be the label of each instance. Let $p(x,y)$ be the underlying joint distribution density of $\{x,y\}$. In supervised OR problems, the labeled datasets of each class can be viewed as drawn from the conditional distributional density $p(x|y)$ as follows,
\begin{align}
\mathcal{D}_j = \{x_i^j\}_{i=1}^{n_j} \sim p(x|y=j),\;j=1,\cdots, k.\nonumber
\end{align}

Generally speaking, the vast majority of existing ordinal regression models can be represented as ,
\begin{equation}
h(x)=\left\{\begin{matrix}
1, & \mathrm{if}& f(x)<b_1\\
j, & \mathrm{if}& b_{j-1}<f(x)<b_j, j=2,\cdots,k-1\\
k, & \mathrm{if}& f(x)>b_{k-1}
\end{matrix}\right.,\nonumber
\end{equation}
where $b_1<\cdots<b_{k-1}$ denote the thresholds and $f:\mathbb{R}^d\mapsto \mathbb{R}$ is commonly referred as a ranking function \cite{waegeman2008roc}. The model $h$ means that we need to consider $k-1$ parallel hyperplanes, $f(x)-b_j$, which decompose the ordinal target variables into $k-1$ binary classification subproblems. Therefore, the problem of calculating the AUC in OR problems can be transformed to that of calculating AUC in $k-1$ binary subproblems. 

In binary classification, AUC means the probability that a randomly sampled positive instance receive a higher ranking than a randomly drawn negative instance. Thus, to calculate AUC in $j$-th subproblem, we need to define which part is positive. Fortunately, in OR problems, instances can naturally be ranked by their ordinal labels. Therefore, for the $j$-th binary classification hyperplane, the first consecutive $j$ categories, $1,\cdots,j$, can be regarded as negative, and the rest of the classes, ${j+1},\cdots,k$, can be regarded as positive. Then we obtain two new datasets as follows,
\begin{equation}
	\mathcal{D}_n^j = \mathcal{D}_1\cup\cdots\cup\mathcal{D}_j\sim p^j_{-}=\dfrac{\sum_{i=1}^{j}\theta_ip(x|y=i)}{\sum_{i=1}^{j}\theta_i},\nonumber
\end{equation}
\begin{equation}
	\mathcal{D}_p^j = \mathcal{D}_{j+1}\cup\cdots\cup\mathcal{D}_k\sim p^j_{+}=\dfrac{\sum_{i=j+1}^{k}\theta_ip(x|y=i)}{\sum_{i=j+1}^{k}\theta_i},\nonumber
\end{equation}
where $\theta_i$ denotes class prior of each class.
Then AUC in each binary subproblem can be calculated by 
\begin{equation}
	\mathrm{AUC} = 1-\mathbb{E}_{x_p^j \sim p^j_{+}}\left[\mathbb{E}_{x_n^j\sim p^j_{-}}\left[l_{01}\left(f(x_p^j),f(x_n^j)\right)\right]\right],\nonumber
\end{equation}
where $l_{01}(u,v)=\dfrac{1}{2}\left(1-\mathrm{sign}(u-v)\right)$ and $\mathbb{E}_{x\sim p(\cdot)}$ denotes the expectation over distribution $p(\cdot)$. The zero-one loss can be replaced by squared pairwise loss function $l_s(u,v)=(1-u+v)^2$ \cite{gao2015consistency,gao2013one}. While in real-world problems, the distribution is unknown and one usually uses the empirical mean to replace the expectation. Thus, the second term can be rewritten as 
\begin{align}\label{PN_AUC_risk}
R^j_{\mathrm{PN}}
&= \mathbb{E}_{x_p^j \in \mathcal{D}_p^j}\left[\mathbb{E}_{x_n^j\in\mathcal{D}_n^j}\left[l_{01}\left(f(x_p^j),f(x_n^j)\right)\right]\right],
\end{align}
where $\mathbb{E}_{x\in \mathcal{D}}$ denotes the empirical mean on the dataset $\mathcal{D}$. Equation (\ref{PN_AUC_risk}) can be viewed as AUC risk between positive and negative instances. Obviously, maximizing AUC is equivalent to minimizing AUC risk $R_{\mathrm{PN}}$.

According to \cite{waegeman2008roc}, the goal of AUC optimization in OR problems is to train a ranking function $f$ which can minimize the overall AUC risk of $k-1$ subproblems,
\begin{align}\label{supervised_or_auc}
R_{\alpha} = \dfrac{1}{k-1}\sum_{j=1}^{k-1} R^j_{\mathrm{PN}}.
\end{align}

\subsection{Semi-Supervised Ordinal Regression AUC Optimization}\label{sec:objection}
In semi-supervised OR problems, the unlabeled data can be viewed as drawn from the marginal distribution $p(x)$ as follows,
\begin{equation}
\mathcal{D}_u = \{x_i^u\}_{i=1}^{n_u} \sim p(x),
\end{equation}
where $p(x) = \sum_{j=1}^{k}\theta_jp(x|y=j)$. For the $j$-th subproblem, the unlabeled data can be viewed as drawn from distribution $p^j(x)= \pi^jp^j_{+}+(1-\pi^j)p^j_{-}$, where $\pi^j=\sum_{i=j+1}^{k}\theta_i$. 

The key idea to incorporate the unlabeled instances into the binary AUC optimization process is to treat the unlabeled instances as negative and then compare them with positive instances; treat them as positive and compare them with negative data \cite{wang2015optimizing}. Thus, the AUC risk $R^j_{\mathrm{PU}}$ between positive and unlabeled instances and the AUC risk $R^j_{\mathrm{NU}}$ between unlabeled and negative instances can be defined as follow,
\begin{align}
R_{\mathrm{PU}}^j &= \mathbb{E}_{x_p^j \in \mathcal{D}_p^j}\left[\mathbb{E}_{x_u^j\in \mathcal{D}_u}\left[l_{01}\left(f(x_p^j),f(x_u^j)\right)\right]\right], \\
R_{\mathrm{NU}}^j &= \mathbb{E}_{x_u^j \in \mathcal{D}_u}\left[\mathbb{E}_{x_n^j \in \mathcal{D}_n^j}\left[l_{01}\left(f(x_u^j),f(x_n^j)\right)\right]\right],
\end{align}
Xie and Li, \shortcite{xie2018semi} have shown that $R^j_{\mathrm{PU}}$ and $R^j_{\mathrm{NU}}$ are equivalent to $R^j_{\mathrm{PN}}$ with a linear transformation as follows,
\begin{align}
R^j_{\mathrm{PN}} = R^j_{\mathrm{PU}}+R^j_{\mathrm{NU}}-\dfrac{1}{2}.
\end{align}
Thus, the AUC risk $R^j_{\mathrm{PNU}}$ for the $j$-th binary semi-supervised problem is 
\begin{equation}
R^j_{\mathrm{PNU}} = \gamma^jR^j_{\mathrm{PN}}+(1-\gamma^j)\left( R^j_{\mathrm{PU}}+R^j_{\mathrm{NU}}-\dfrac{1}{2}\right),
\end{equation} 
where the first term is the AUC risk computed from the labeled instances only, the second term is an estimation AUC risk using both labeled and unlabeled instances and $\gamma^j$ is trade-off parameter.
Similar to Equation (\ref{supervised_or_auc}), the overall AUC risk for the $k-1$ hyperplanes in the S$^2$OR problem can be formulated as follows,
\begin{equation}\label{ss_auc_risk}
R_{\mu}=\dfrac{1}{k-1}\sum_{j=1}^{k-1}R^j_{\mathrm{PNU}}.
\end{equation}

To avoid overfitting caused by directly minimizing Equation (\ref{ss_auc_risk}), a regularization term is usually added as follows,
\begin{align}\label{regular_auc}
\mathcal{L}(f) =  \dfrac{\lambda}{2}\parallel f\parallel_{\mathcal{H}}^2+\dfrac{1}{k-1}\sum_{j=1}^{k-1}R^j_{\mathrm{PNU}},
\end{align}
where $\parallel \cdot \parallel_{\mathcal{H}}$ denotes the norm in RKHS $\mathcal{H}$, $\lambda>0$ is regularization parameter .

\subsection{Random Fourier Feature}
For any \textit{continuous}, \textit{real-valued}, \textit{symmetric} and \textit{shift-invariant} kernel function $k(x,x')$, according to Bochner Theorem \cite{rudin2017fourier}, there exists a nonnegative Fourier transform function as $k(x,x') = \int_{\mathbb{R}^d} p(\omega)e^{j\omega^T(x-x')}d\omega$, where $p(w)$ is a density function associated with $k(x,x')$. The integrand $e^{j\omega^T(x-x')}$ can be replaced with $\cos\omega^T(x-x')$ \cite{rahimi2008random}. Thus, the feature map for $m$ random features of $k(x,x')$ can be formulated as follows.
\begin{eqnarray}
\phi_{\omega}(x) = \sqrt{1/D}[\cos (\omega_1^Tx),\cdots,\cos (\omega_m^Tx),\nonumber\\ \sin(\omega_1^Tx),\cdots,\sin(\omega_m^Tx) ]^T ,\nonumber
\end{eqnarray}
where $\omega_i$ is randomly sampled according to the density function $p(\omega)$.
Obviously, $\phi^T_{\omega}(x)\phi_{\omega}(x')$ is an unbiased estimate of $k(x,x')$.

\section{Quadruply Stochastic Gradient Method}

Based on the definition of the ranking function $f\in\mathcal{H}$, we can obtain $\nabla f(x) =k(x,\cdot)$, and $\nabla \parallel f \parallel_{\mathcal{H}}^2=2f$. To calculate the gradient of objective function, we use the squared pairwise loss $l_s(u,v)$ function to replace zero-one loss $l_{01}(u,v)$. Then we can obtain the full gradient of our  objective function w.r.t. $f$ as follows, 
\begin{align}
\nabla \mathcal{L}&=\dfrac{1}{k-1}\sum_{j=1}^{k-1}(\gamma^j\mathbb{E}_{x_p^j\in\mathcal{D}_p^j}[\mathbb{E}_{x_n^j\in\mathcal{D}_n^j}[l_1'(f(x_p^j),f(x_n^j))k(x_p^j,\cdot)\nonumber\\
&+l_2'(f(x_p^j),f(x_n^j))k(x_n^j,\cdot)]]\nonumber\\
&+(1-\gamma^j)(\mathbb{E}_{x_p^j\in\mathcal{D}_p^j}[\mathbb{E}_{x_u^j\in\mathcal{D}_u}[l_1'(f(x_p^j),f(x_u^j))k(x_p^j,\cdot)\nonumber\\&+l_2'(f(x_p^j),f(x_u^j))k(x_u^j,\cdot)]]\nonumber\\
&+\mathbb{E}_{x_u^j\in\mathcal{D}_u}[\mathbb{E}_{x_n^j\in\mathcal{D}_n^j}[l_1'(f(x_u^j),f(x_n^j))k(x_u^j,\cdot)\nonumber\\&+l_2'(f(x_u^j),f(x_n^j))k(x_n^j,\cdot)]]))+\lambda f \nonumber
\end{align}
where $l_1'(u,v)$ denotes the derivative of $l_s(u,v)$ w.r.t. the first argument in the functional space, $l_2'(u,v)$ denotes the derivative of $l_s(u,v)$ w.r.t. the second argument in the functional space.
\subsection{Stochastic Functional Gradients}
Directly calculating the full gradient is time-consuming. In order to reduce the computational complexity, we update the ranking function using a quadruply stochastic framework. For each subproblem, we randomly sample a positive instance $x^j_p$ from $\mathcal{D}_p^j$, a negative instance $x^j_n$ from $\mathcal{D}_n^j$ and an unlabeled instance $x_u$ from $\mathcal{D}_u$ in each iteration. 

For convenience, we use $l_i^j$, $i=1,\cdots,6$, to denote the abbreviation of $l_1'(f(x_p^j),f(x_n^j))$, $l_2'(f(x_p^j),f(x_n^j))$, $l_1'(f(x_p^j),f(x_u^j))$, $l_2'(f(x_p^j),f(x_u^j))$, $l_1'(f(x_u^j),f(x_n^j))$, $l_2'(f(x_u^j),f(x_n^j))$ in $j$-th subproblem, respectively. Then the stochastic gradient of Equation (\ref{ss_auc_risk}) w.r.t $f$ can be calculated by using these random instances,    
\begin{align}\label{ss_stochastic_gradient}
\xi(\cdot) &= \dfrac{1}{k-1}\sum_{j=1}^{k-1}(\gamma^j(l_1^jk(x_p^j,\cdot)+l_2^jk(x_n^j,\cdot))\nonumber\\
&+(1-\gamma^j)(l_3^jk(x_p^j,\cdot)+l_4^jk(x_u^j,\cdot)\nonumber\\
&+l_5^jk(x_u^j,\cdot)+l_6^jk(x_n^j,\cdot)))
\end{align}

\subsection{Kernel Approximation}
When calculating the gradient $\xi(\cdot)$, we still need to calculate the kernel matrix. In order to further reduce the complexity, we introduce random Fourier features into gradient $\xi(\cdot)$. Then we can obtain the following approximated gradient,
\begin{align}
\zeta(\cdot) &= \dfrac{1}{k-1}\sum_{j=1}^{k-1}(\gamma^j(l_1^j\phi_{\omega}(x_p^j)\phi_{\omega}(\cdot)+l_2^j\phi_{\omega}(x_n^j)\phi_{\omega}(\cdot))\nonumber\\
&+(1-\gamma^j)(l_3^j\phi_{\omega}(x_p^j)\phi_{\omega}(\cdot)+l_4^j\phi_{\omega}(x_u^j)\phi_{\omega}(x)\nonumber\\
&+l_5^j\phi_{\omega}(x_u^j)\phi_{\omega}(\cdot)+l_6^j\phi_{\omega}(x_n^j)\phi_{\omega}(\cdot)))
\end{align}
Obviously, we have $\xi(\cdot)=\mathbb{E}_{\omega}[\zeta(\cdot)]$. Besides, since four sources of randomness of each subproblem, \textit{$x_p^j$, $x_n^j$, $x_u^j$, $\omega$}, are involved in calculating gradient $\zeta(\cdot)$, we can denote the approximated gradient $\zeta(\cdot)$ as quadruply stochastic functional gradient.

\subsection{Update Rules}
For convenience, we denote the function value as $h(x)$ while using the real gradient $\xi(\cdot)$, and $f(x)$ while using the approximated gradient $\zeta(\cdot)$. Obviously, $h(x)$ is always in the RKHS $\mathcal{H}$ while $f(x)$ may be outside $\mathcal{H}$.
We give the update rules of $h(\cdot)$ as follows,
\begin{align}
h_{t+1}(\cdot)&=h_t(\cdot)-\eta_t\nabla\mathcal{L}(h)= \sum_{i=1}^{t}a^i_t\xi^i(\cdot), \quad\forall \; t>1\nonumber
\end{align}
where $\eta_t$ denotes the step size in $t$-th iteration, $a_t^i = -\eta_t\prod_{j=i+1}^{t}(1-\eta_j\lambda)$ and $h_0^1(\cdot)=0$. 

Since $\zeta(\cdot)$ is an unbiased estimate of $\xi(\cdot)$, they have the similar update rules. Thus, the update rule by using $\zeta_j(\cdot)$ is
\begin{align}
f_{t+1}(\cdot)&=f_t(\cdot)-\eta_t\nabla\mathcal{L}(f)= \sum_{i=1}^{t}a^i_t\zeta^i(\cdot), \quad\forall \; t>1\nonumber
\end{align}
where $f_0^1(\cdot)=0$. 

In order to implement the algorithm in a computer program, we introduce sequences of constantly-changing coefficients $\{\alpha_i\}_{i=1}^t$. Then the update rules can be rewritten as 
\begin{align}
f(x) &= \sum_{i=1}^{t}\alpha_i\phi_{\omega}(x),\\
\alpha_i &= -\dfrac{\eta_i}{k-1}\sum_{j=1}^{k-1}(\gamma^j(l_1^j\phi_w(x^j_p)+l_2^j\phi_w(x^j_n))\nonumber\\&+(1-\gamma^j)(l_3^j\phi_w(x^j_p)+l_4^j\phi_w(x^j_u)\nonumber\\&+l_5^j\phi_w(x^j_n)+l_6^j\phi_w(x^j_u))),\\
\alpha_j &=(1-\eta_i\lambda)\alpha_j, \; for \; j=1,\cdots,i-1,
\end{align}
\subsection{Calculate the Thresholds}
Since the thresholds are ignored in AUC optimization, an additional strategy is required to calculate them. As all the function values $f(x)$ of labeled instances are already known, the thresholds can be calculated by minimizing the following equation, which penalizes every erroneous threshold of all the binary subproblems \cite{fathony2017adversarial}.
\begin{equation}
	\min_{\textbf{b}}\mathcal{L}_{AT}=\sum_{i=1}^{n_l}\left(\sum_{j=1}^{y_i-1}\delta(f(x_i)-b_j)+\sum_{j=y_i}^{k}\delta(b_j-f(x_i))\right), \nonumber
\end{equation}
where $n_l$ denotes the number of labeled instances, $\delta(\cdot)$ denotes the surrogate loss functions and $b_k = \infty$. Obviously, it is a Linear Programming problem and can be easily solved. Besides, the solution has following property (Proof in Appendix). 
\begin{lemma}
	Let $\textbf{b}^*=[b_1^*,\cdots,b_{k-1}^*]$ be the optimal solution, we have that $\textbf{b}^*$ is unique and $b_1^*<\cdots<b_{k-1}^*$.
\end{lemma}
\subsection{Algorithms}
The overall algorithms for training and prediction are summarized in Algorithm \ref{alg:train} and \ref{alg:predict}. Instead of saving all the random features, we following the \textit{pseudo random number generator} setting of \cite{dai2014scalable} with seed $i$ to generate random features in each iteration. We only need to save the seed $i$ and keep it aligned between training and prediction, then we can regenerate the same random features. We also use the coefficients to speed up calculating the function value. Specifically, each iteration of the training algorithm executes the following steps.
\begin{enumerate}
	
	\item \textit{Randomly Sample Data Instances:} 
	We can randomly sample a batch instances from class $i,\cdots,k$ and unlabeled dataset respectively, and then conduct the data of $k-1$ subproblems instead of sampling instances for each subproblem.
	
	\item \textit{Approximate the Kernel Function:} \ Sample $\omega_i \sim p(\omega)$ with random seed $i$ to calculate the random features on-the-fly. We keep this seed aligned between prediction and training to regenerate the same random features.
	
	\item \textit{Update Coefficients:} \  We compute the current coefficient $\alpha_i$ in $i$-th loop, and then update the former coefficients $\alpha_j$ for $j=1,\cdots,i-1$.
	
\end{enumerate}

\begin{algorithm}[!ht]	
	\caption{\textbf{ QS$^3$ORAO}} 
	\renewcommand{\algorithmicrequire}{\textbf{Input:}}
	\renewcommand{\algorithmicensure}{\textbf{Output:}}
	\begin{algorithmic}[1] 
		\REQUIRE $p(\omega)$, $\phi_{\omega}(x)$, $l(u,v)$, $ \lambda$, $\gamma_i$, $\sigma$, $k$, $t$.
		\ENSURE $\{\alpha_i\}_{i=1}^t$, $b_1,\cdots,b_{k-1}$
		\FOR{$i=1,...,t$}
		\FOR{$j=1,\cdots,k-1$}
		\STATE Sample $x^j_p$ from $\mathcal{D}_{j+1}\cup\cdots\cup\mathcal{D}_{k}$.
		\STATE Sample $x^j_n$ from $\mathcal{D}_{1}\cup\cdots\cup\mathcal{D}_j$.
		\STATE Sample $x^j_u \sim \mathcal{D}_u$.
		\ENDFOR		
		
		\STATE Sample $\omega_i \sim p(\omega)$ with seed $i$.
		\FOR{$j=1,\cdots,k-1$}
		\STATE$f(x_i)=$\textbf{Predict}$(x_i,\{\alpha_j\}_{j=1}^{i-1},\{\beta_j\}_{j=1}^{i-1})$.
		\ENDFOR
		\STATE $\alpha_i =-\dfrac{\eta_i}{k-1}\sum_{j=1}^{k-1}(\gamma(l_1^j\phi_w(x^j_p)+l_2^j\phi_w(x^j_n))+(1-\gamma)(l_3^j\phi_w(x^j_p)+l_4^j\phi_w(x^j_u)+l_5^j\phi_w(x^j_n)+l_6^j\phi_w(x^j_u)))$	
		\STATE $\alpha_j = (1-\eta_j\lambda)\alpha_j \;for\; j=1,...,i-1$.
		\ENDFOR
		
		\STATE Minimize $\mathcal{L}_{AT}$ to find threshold $b_1,\cdots,b_{k-1}$.			
	\end{algorithmic}
	\label{alg:train}
\end{algorithm}

\begin{algorithm}[!ht]
	\caption{$f(x)=$\textbf{Predict}$(x,\{\alpha_i\}_{i=1}^t)$} 
	\renewcommand{\algorithmicrequire}{\textbf{Input:}}
	\renewcommand{\algorithmicensure}{\textbf{Output:}}
	\begin{algorithmic}[1] 
		\REQUIRE $p(\omega),\phi_{\omega}(x)$
		\ENSURE $f(x)$
		\STATE Set $f(x) = 0$.
		\FOR{$i=1,...,t$}
		\STATE Sample $\omega_i \sim p(\omega)$ with seed $i$.
		\STATE $f(x)=f(x)+\alpha_i\phi_{\omega}(x)$
		\ENDFOR
	\end{algorithmic}
	\label{alg:predict}
\end{algorithm}

\section{Convergence Analysis}

In this section, we prove that QS$^3$ORAO converges to the optimal solution at the rate of $O(1/t)$. We first give several assumptions which are common in theoretical analysis.
\begin{assumption}
	There exists an optimal solution $f^{*}$ to the problem (\ref{regular_auc}).
\end{assumption}
\begin{assumption}
	(Lipschitz continuous). The first order derivative of $l_s(u,v)$ is $L_1$-\textbf{Lipschitz continous} in terms of $u$ and $L_2$-\textbf{Lipschitz continous} in terms of $v$.
\end{assumption} 

\begin{assumption}
	(Bound of derivative). Assume that, we have $|l_1'(u,v)|\leq M_1$ and $|l_2'(u,v)| \leq M_2$, where $M_1>0$ and $M_2>0$.
\end{assumption}


\begin{assumption}
	(Bound of kernel function). The kernel function is bounded, \textit{i.e.,} ${k}(x,x')\leq \kappa$, where $\kappa>0$.
\end{assumption}

\begin{assumption}
	(Bound of random features norm). The random features norm is bounded, \textit{i.e.,} $|\phi_{\omega}(x)\phi_{\omega}(x')|\leq \phi$.
\end{assumption}

Then we prove that $f_t$ can converge to the optimal solution $f^*$ based on the framework in \cite{dai2014scalable}. Since $f_t$ may outside the RKHS, we use $h_t$ as an intermediate value to decompose the error between $f_t$ and $f^*$:
\begin{equation}
|f_t(x)-f^*(x)|^2\leq 2|f_t(x)-h_t(x)|^2+2\kappa \parallel h_t - f^* \parallel_{\mathcal{H}}^2,\nonumber
\end{equation}
where the first term can be regarded as the error caused by random features and the second term can be regarded as the error caused by randomly sampling data instances. We first give the bound of these two errors in Lemma \ref{lemma:error_of_random_feature} and Lemma \ref{lemma:error_due_to_random_data} respectively. All the detailed proofs are in Appendix.
\begin{lemma}[Error due to random features]\label{lemma:error_of_random_feature}
	Assume $\chi$ denote the whole training set. For any $x\in \chi$, we have
	\begin{align}
	\mathbb{E}_{x_p^t,x_n^t,x_u^t,w_t}\left[ \left| f_t(x)-h_t(x)\right| \right] \leq B^2_{1,t+1},
	\end{align}
	where $B^2_{1,t+1}:=M^2({\kappa}+{\phi})^2\sum_{i=1}^{t}|a_t^i|$, $M=\dfrac{1}{k-1}\sum_{j=1}^{k-1}(2-\gamma^j)(M_1+M_2)$, and $B_{1,1}=0$.
\end{lemma}

\begin{lemma}\label{lemma:upper_bound}		
	Suppose $\eta_i = \dfrac{\theta}{i}$ ($1 \leq  i \leq t$) and $\theta\lambda \in (1,2)\cup\mathbb{Z}_{+}$. We have $|a_t^i| \leq \dfrac{\theta}{t}$ and $\sum_{i=1}^t |a_t^i|^2 \leq \dfrac{\theta^2}{t}$.
\end{lemma}
\begin{remark}
	According to  Lemmas \ref{lemma:error_of_random_feature}  and \ref{lemma:upper_bound}, the error caused by random features has the convergence rate of $O(1/t)$ with proper learning rate and $\theta\lambda \in (1,2)$.
\end{remark}
\begin{lemma}[Error due to random data]\label{lemma:error_due_to_random_data}
	Set $\eta_t = \dfrac{\theta}{t}$, $\theta>0$, such that $\theta\lambda \in (1,2)\cup \mathbb{Z}_{+}$, we have
	\begin{equation}		
	\mathbb{E}_{x_p^t,x_n^t,x_u^t,\omega_t}\left[\| h_{t+1} - f^{*} \|_{\mathcal{H}}^2\right] \leq \dfrac{Q_1^2}{t},		
	\end{equation}
	where
	$Q_1 = \max\left\{\| f^{*} \|_{\mathcal{H}},\frac{Q_0 + \sqrt{Q_0^2+(2\theta\lambda-1)(1+\theta\lambda)^2\theta^2\kappa M^2}}{2\theta\lambda-1} \right\} $,
	$Q_0 = \sqrt{2}\kappa^{1/2}(\kappa+\phi)LM\theta^2$ and $L = \dfrac{1}{k-1}\sum_{j=1}^{k-1}(2-\gamma^j)(L_1+L_2)$.
\end{lemma}

According to Lemma \ref{lemma:error_of_random_feature} and Lemma \ref{lemma:error_due_to_random_data}, we can bound the error between $f_t$ and $f^*$.
\begin{theorem}[Convergence in expectation]\label{theorem:Convergence_in_expectation}
	Let $\chi$ denote the whole training set in semi-supervised learning problem. Set $\eta_t = \dfrac{\theta}{t}$, $\theta >0$, such that $\theta\lambda \in (1,2)\cup \mathbb{Z}_{+}$. $\forall x \in \chi$, we have
	\begin{align}
	\mathbb{E}_{x_t^p, x_t^n, x_t^u,\omega_t}\left[|f_{t+1}(x)-f^{*}(x)|^2\right] \leq \dfrac{2C^2+2\kappa Q^2_1}{t},  \nonumber
	\end{align}
	where $C^2 = (\kappa+ \phi)^2M^2\theta^2$.
\end{theorem}
\begin{remark}
	Theorem \ref{theorem:Convergence_in_expectation} means that for any given $x$, the evaluated value of $f_{t+1}$ at $x$ will converge to that of $f^*$ at the rate of $O(1/t)$. This rate is the same as that of standard DSG even though our problem is much more complicated and has multiple sources of randomness. 
\end{remark}

\section{Experiments}
\begin{table}	
	\small
	\centering
	\setlength{\tabcolsep}{2.7mm}
	\caption{Datasets used in the experiments.}
	\begin{tabular}{clccc}
		\toprule
		&\textbf{Name}  &\textbf{Features}  & \textbf{Samples}&\textbf{classes} \\
		\hline
		\multirow{4}{*}{\textbf{Discretized
		}} &3D  &       3        &  434,874&5\\		
		&Sgemm 	& 14& 241,600 &5\\
		&Year  &90&463,715 &5\\
		&Yolanda &100&400,000&5\\
		\hline
		\multirow{4}{*}{\textbf{Real-world}}
		&Baby & 1000 & 160,792& 5 \\
		&Beauty & 1000 &  198,502 &5\\
		& Clothes  & 1000 & 278,677&5\\
		& Pet  & 1000 & 157,836 &5\\
		\bottomrule
	\end{tabular}	
	\label{tab:dataset}	
\end{table}

In this section, we present the experimental results on various benchmark and real-world datasets to demonstrate the effectiveness and efficiency of our proposed Q$^3$ORAO.

\begin{figure*}[!t]
	
	\centering
	\begin{subfigure}[b]{0.24\textwidth}
		\centering
		\includegraphics[width=1.45in]{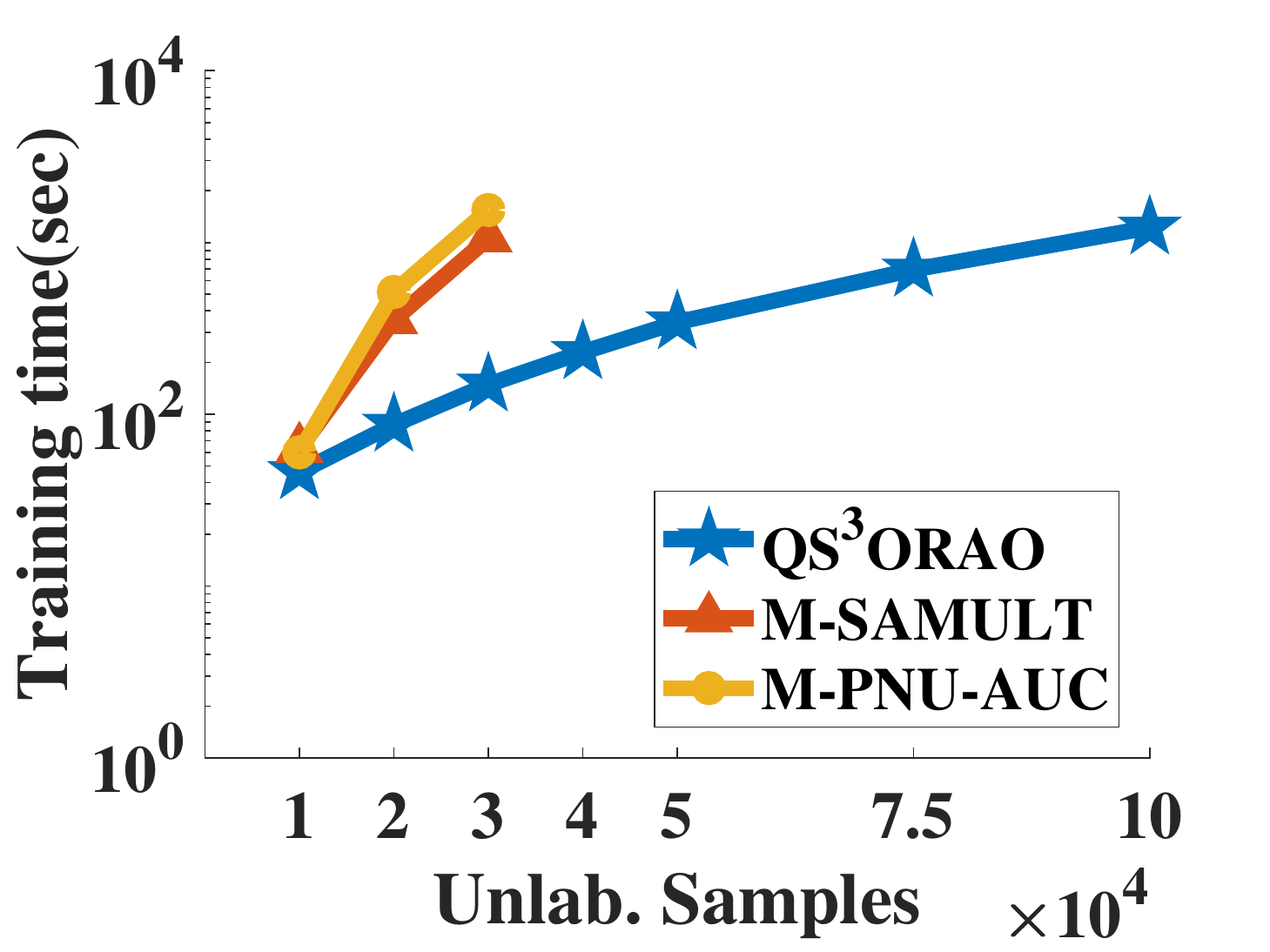}
		\caption{3D}
	\end{subfigure}
	\begin{subfigure}[b]{0.24\textwidth}
		\centering
		\includegraphics[width=1.45in]{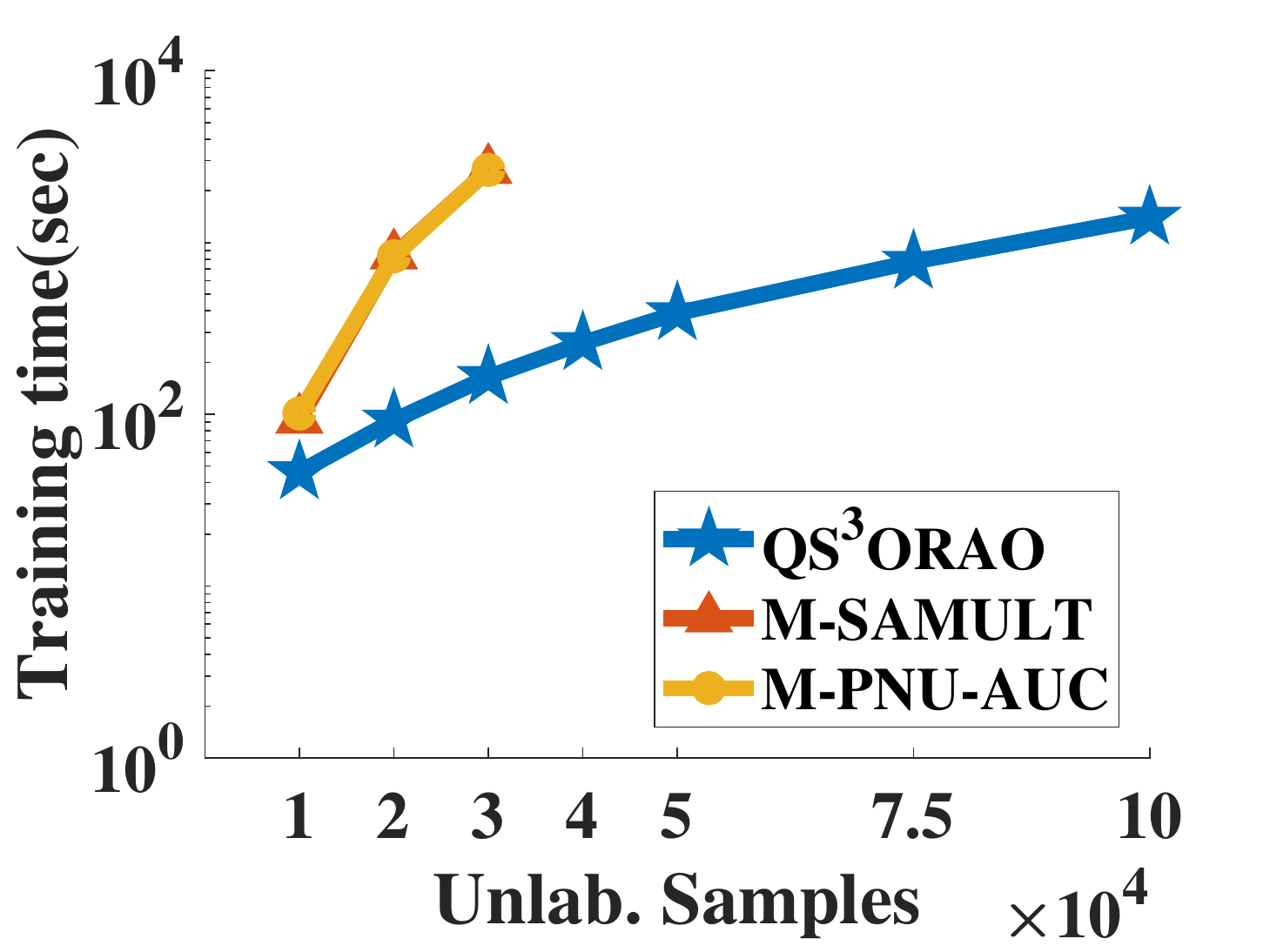}
		\caption{Sgemm}
	\end{subfigure}
	\begin{subfigure}[b]{0.24\textwidth}
		\centering
		\includegraphics[width=1.45in]{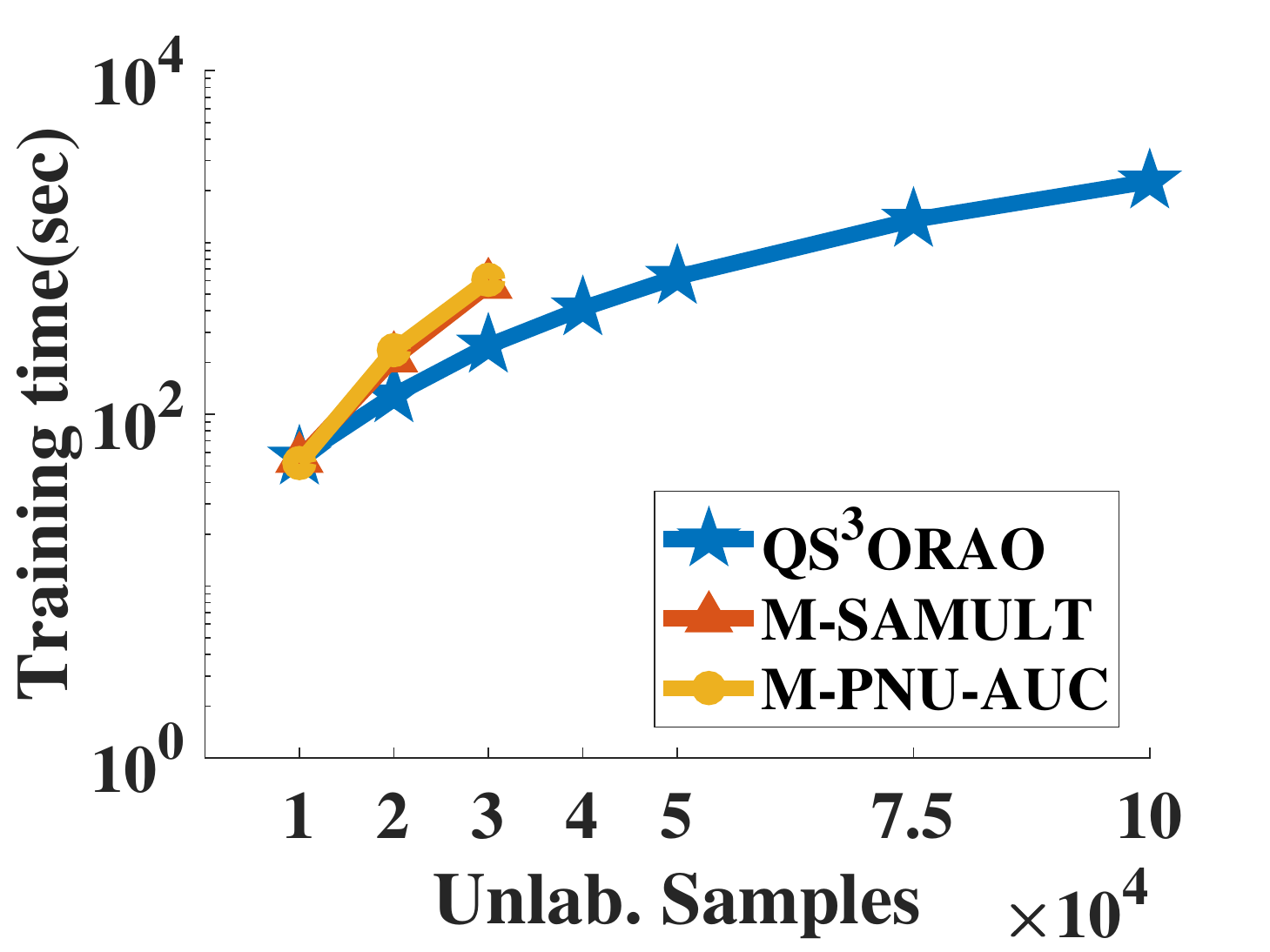}
		\caption{Year}
	\end{subfigure}	
	\begin{subfigure}[b]{0.24\textwidth}
		\centering
		\includegraphics[width=1.45in]{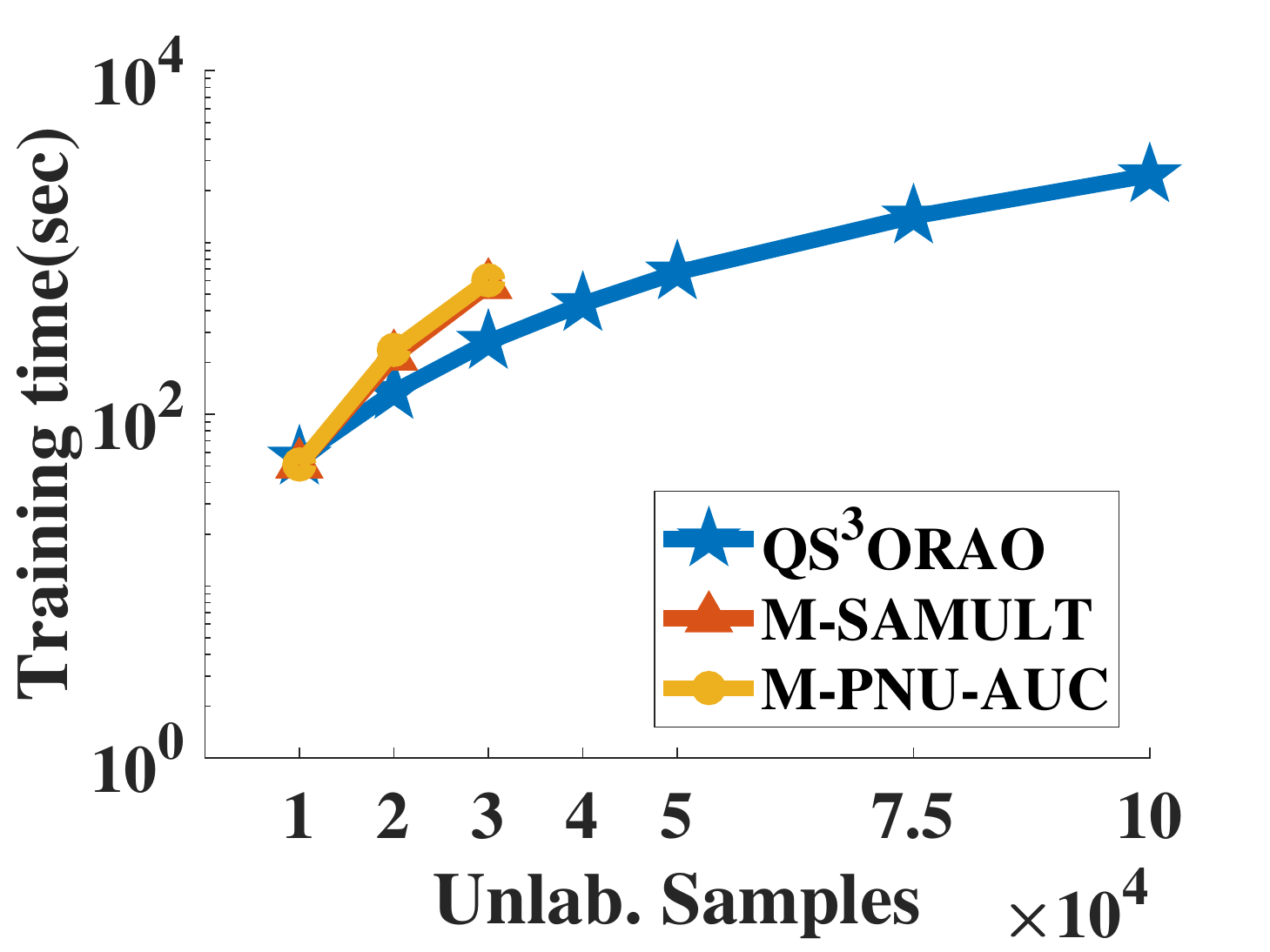}
		\caption{Yolanda}
	\end{subfigure}	
	
	\begin{subfigure}[b]{0.24\textwidth}
		\centering
		\includegraphics[width=1.45in]{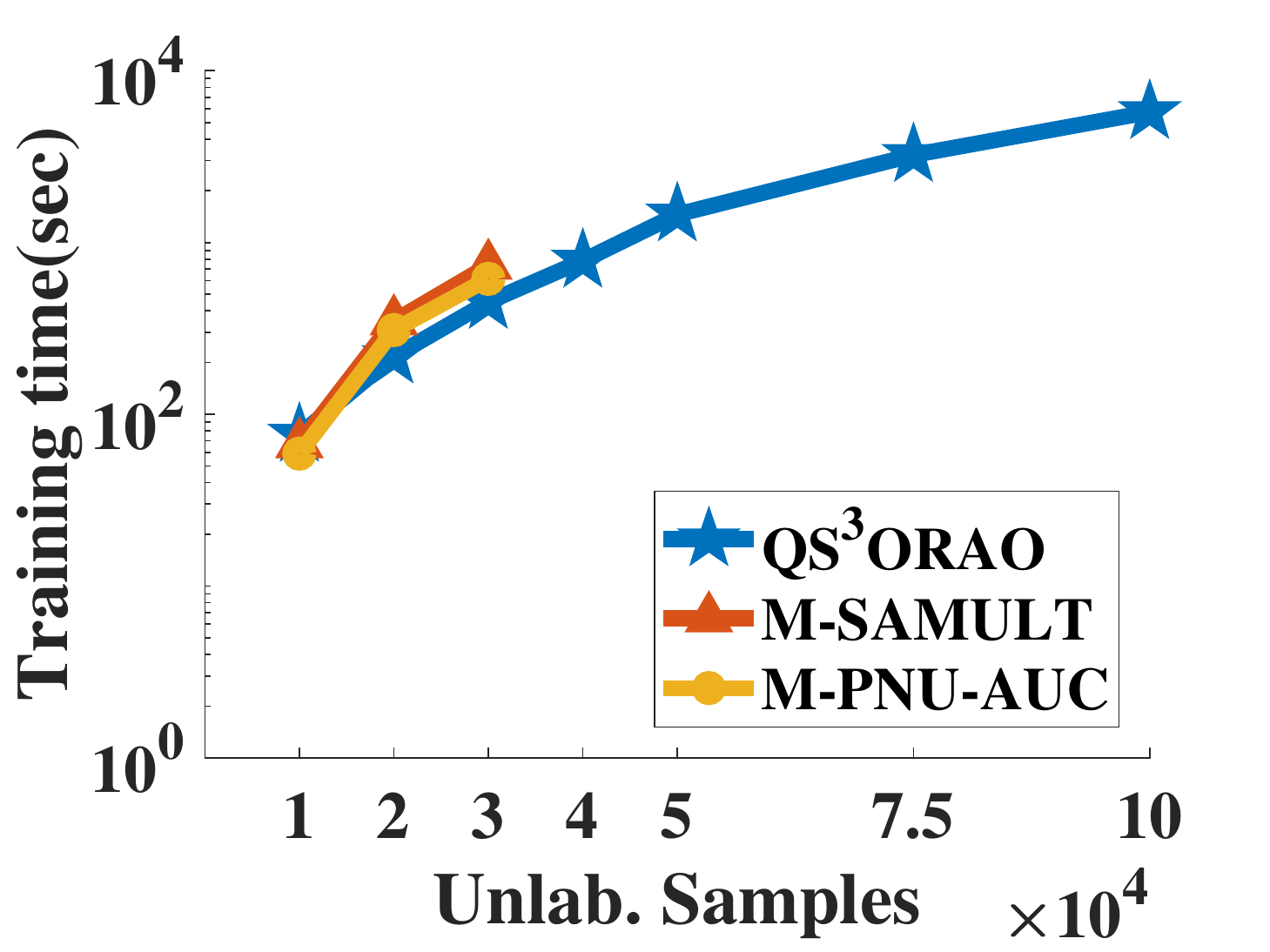}
		\caption{Baby}
	\end{subfigure}	
	\begin{subfigure}[b]{0.24\textwidth}
		\centering
		\includegraphics[width=1.45in]{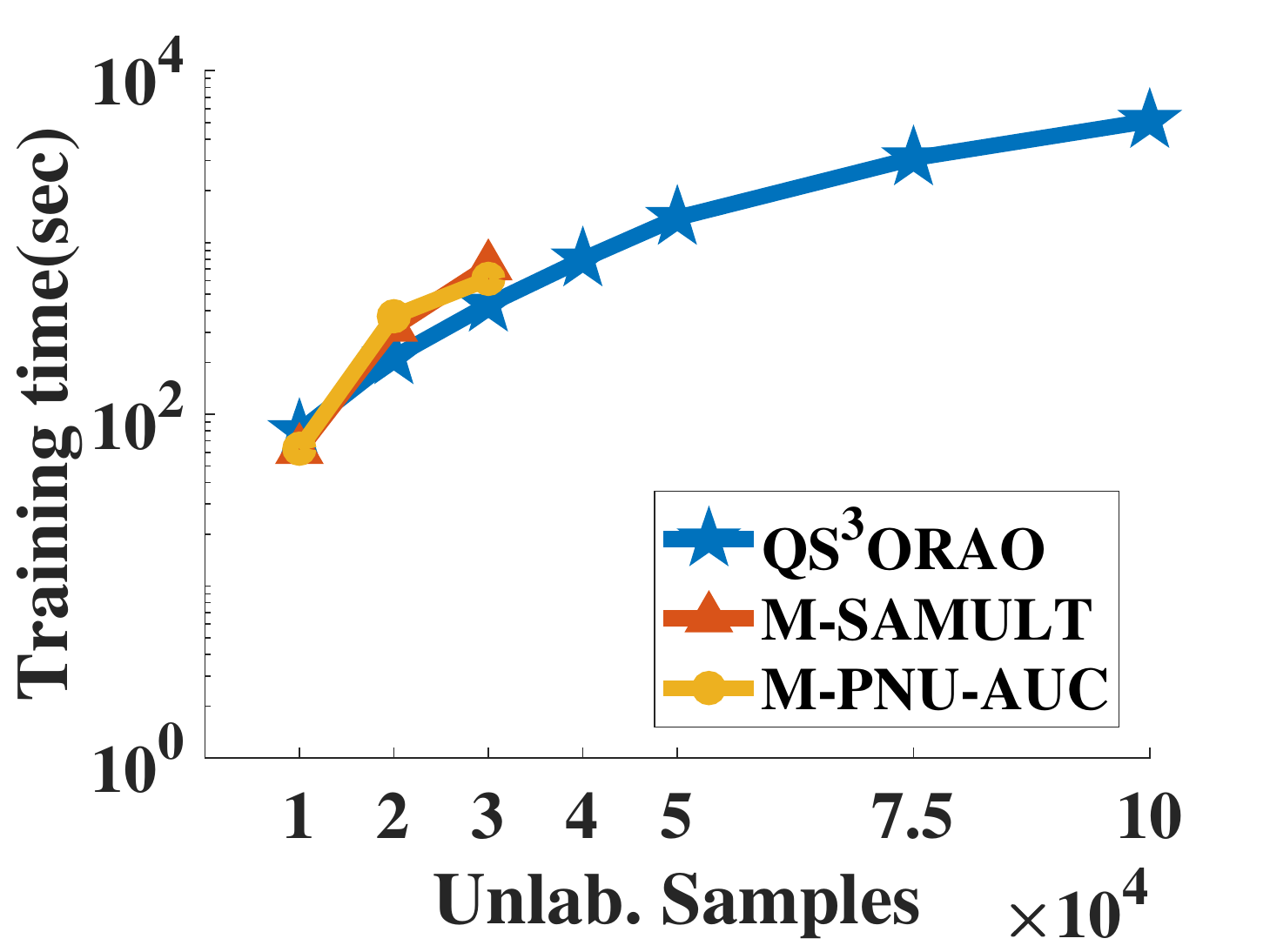}
		\caption{Beauty}
	\end{subfigure}	
	\begin{subfigure}[b]{0.24\textwidth}
		\centering
		\includegraphics[width=1.45in]{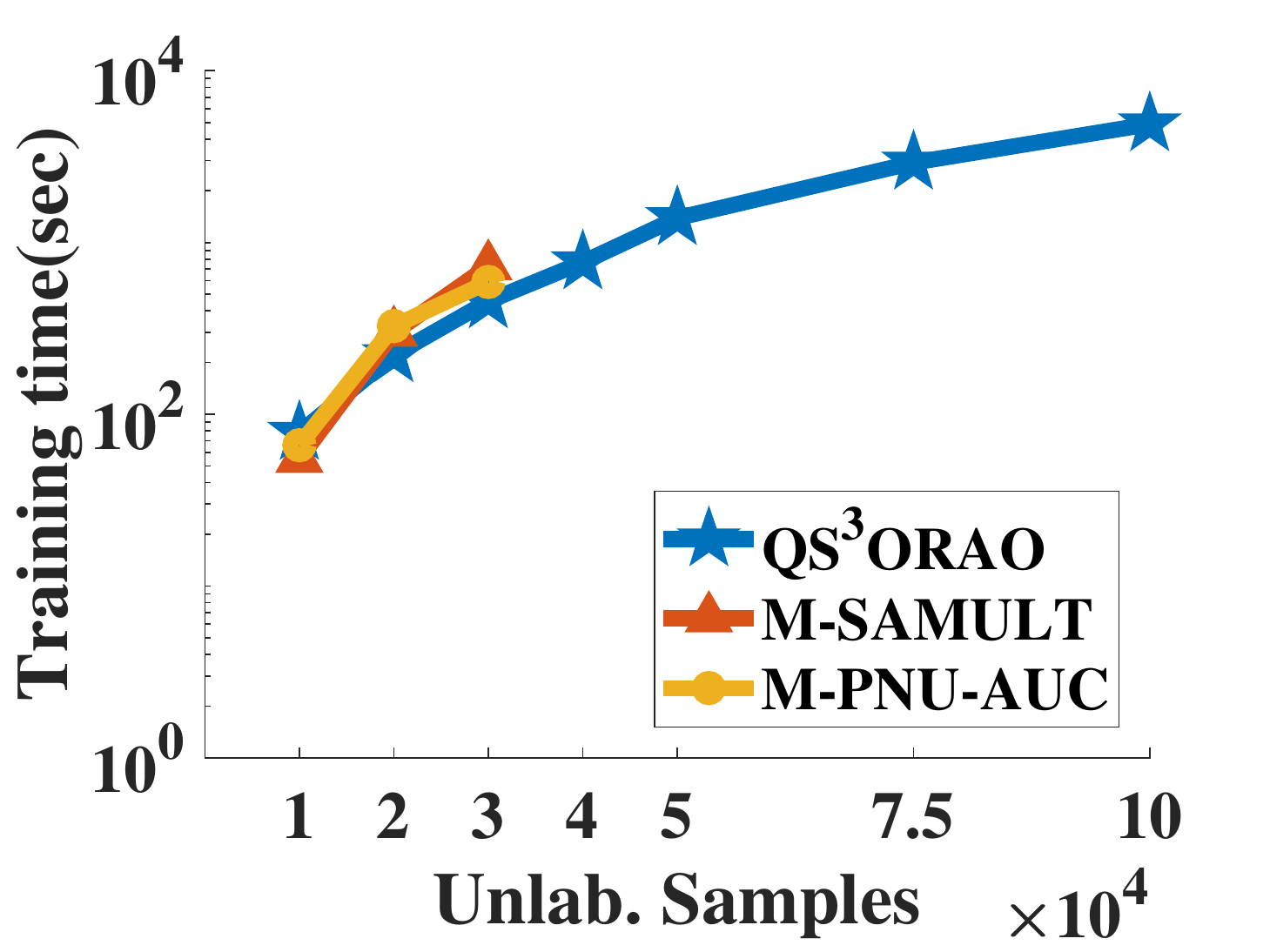}
		\caption{Clothes}
	\end{subfigure}	
	\begin{subfigure}[b]{0.24\textwidth}
		\centering
		\includegraphics[width=1.45in]{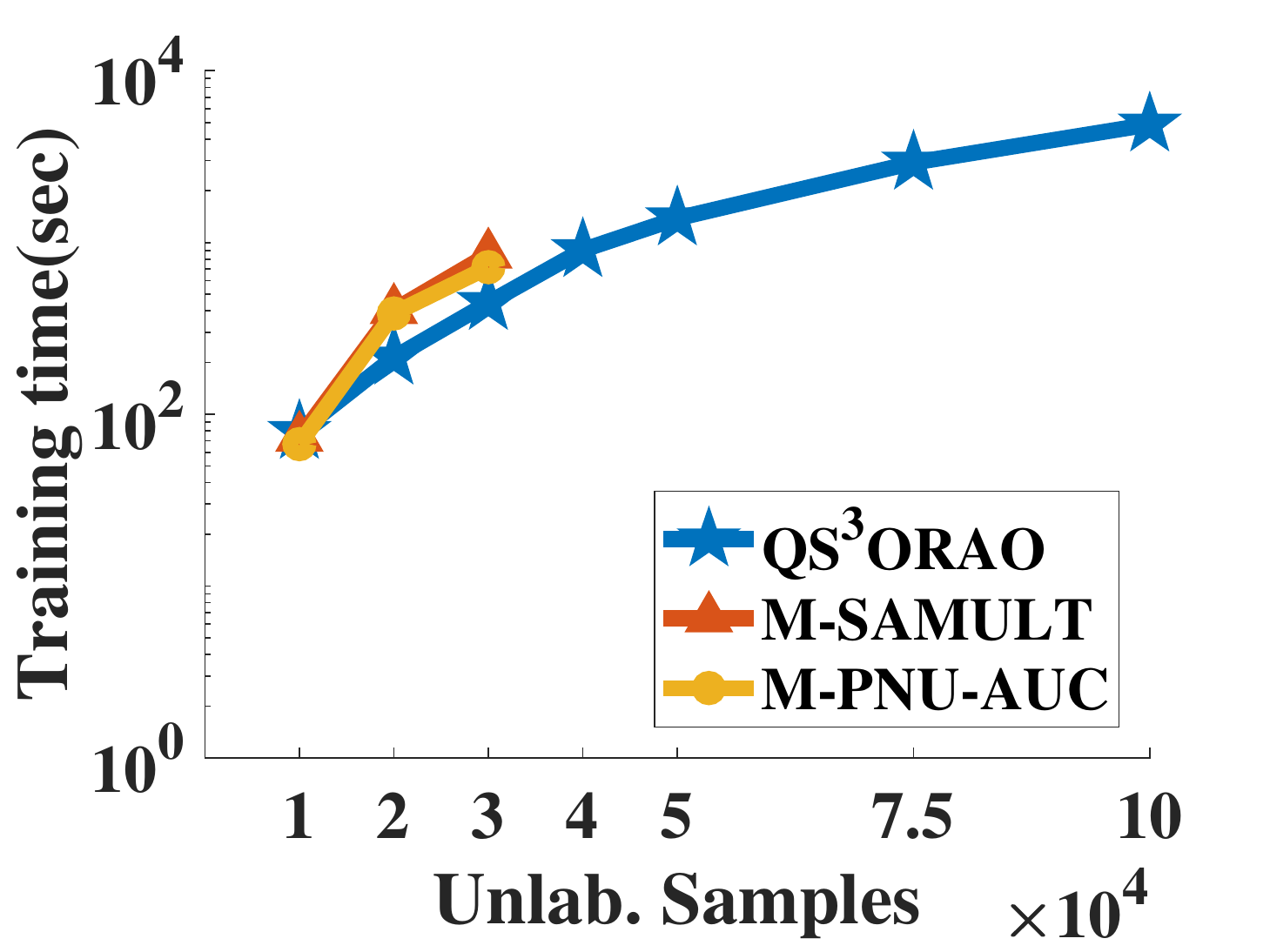}
		\caption{Pet}
	\end{subfigure}	
	
	\caption{The  training time of QS$^3$ORAO, M-SAMULT and M-PNU-AUC  against different sizes of unlabeled samples, where the sizes of labeled samples are fixed at 500. (The lines of M-SAMULT and M-PNU-AUC are incomplete  because their implementations crash on  larger training sets.)}
	\label{fig:time_vs_unlabeled}
\end{figure*}
\begin{figure}[!ht]
	\centering	
	\!\!\!\!\!\!\!\!\!\!\!\includegraphics[width=1.15\columnwidth]{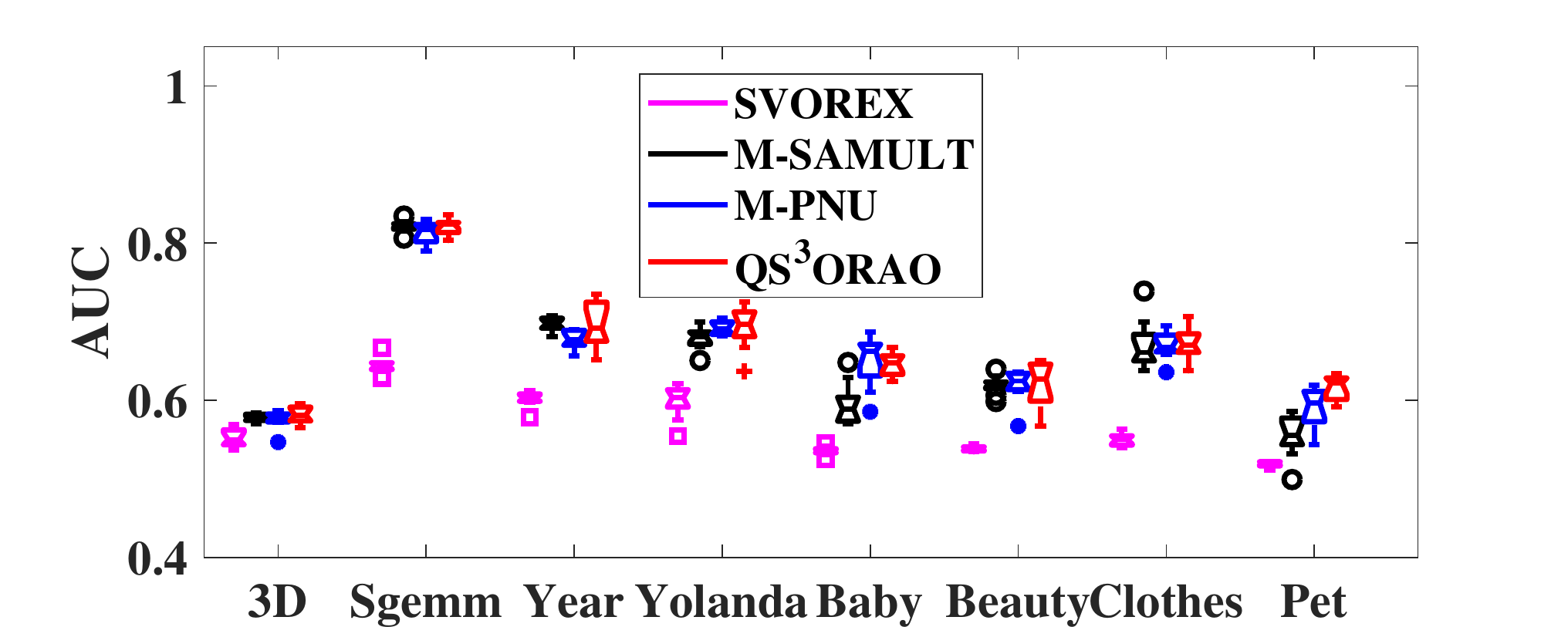}
	\caption{The boxplot of AUC results on unlabeled datasets for SVOREX, M-PNU-AUC, M-SAMULT and our QS$^3$ORAO.}
	\label{fig:auc_result}
\end{figure}
%
\subsection{Experimental Setup}
We compare the AUC results and running time of Q$^3$ORAO with other methods summarized as follows,
\begin{enumerate}
	\item \textit{\textbf{SVOREX}}: Supervised OR algorithm proposed in \cite{chu2007support}.
	\item \textit{\textbf{M-PNU-AUC}}: Multi class version of PNU-AUC \cite{sakai2018semi}, which focuses on binary semi-supervised AUC optimization.
	
	\item \textit{\textbf{M-SAMULT}}: Multi class version of SAMULT \cite{xie2018semi}, which focuses on binary semi-supervised AUC optimization.
\end{enumerate}

We implemented QSG-ORS$^2$AO, SVOREX and SAMULT in MATLAB. We used the MATLAB code from \url{https://github.com/t-sakai-kure/PNU} as the implementation of PNU-AUC. Originally, both PNU-AUC and SAMULT focus on binary semi-supervised AUC optimization problems. We extend them to multi-class version by using a multiclass training paradigm. Specifically, similar to our binary decomposition in our method, we use PNU-AUC and SAMUlT to training $k-1$ classifiers, $f_j(x)$, $j=1,\cdots,k-1$, Then we calculate the average AUC of unlabeled by Equation (\ref{supervised_or_auc}). We denote these multiclass versions as M-PNU-AUC and M-SAMULT. For all algorithms, we use the squared pairwise loss function $l(u,v)=(1-u+v)^2$ and Gaussian kernel $k(x,x') = \exp(\sigma\parallel x-x'\parallel^2)$. The hyper-parameters ($\lambda$ and $\sigma$) were chosen via 5-fold cross-validation from the region $\{(\lambda,\gamma)|2^{-3}\leq \lambda\leq 2^3,2^{-3} \sigma \leq 2^3 \}$. The trade-off parameters $\{\gamma^j\}$ for $k-1$ subproblems were searched from $0$ to $1$ at intervals of $0.1$. 

Note all the experiments were run on a PC with 56 2.2GHz cores and 80GB RAM and all the results are the average of $10$ trials.

\subsection{Datasets}
Table \ref{tab:dataset} summarizes $4$ regression datasets collected from UCI, LIBSVM repositories and $4$ real-world datasets from \textbf{Amazon product datasets}\footnote{\url{http://jmcauley.ucsd.edu/data/amazon/}}. We discretize the regression datasets into equal-frequency bins. For real-world datasets, we first use TF-IDF to process text data, and then reduce the data dimensions to $1000$ by using SVD. To conduct the experiments for semi-supervised problems, we randomly sample 500 labeled instances and drop labels of the rest instances. All the data features are normalized to $[0, 1]$ in advance.

\subsection{Results and Disscussion}
Figure \ref{fig:auc_result} presents the AUC results on the unlabeled dataset of these algorithms. The results show that in most cases, our proposed QS$^3$ORAO has the highest AUC results. In addition, we also compare the AUC results with supervised method SVOREX which uses 500 labeled instances to train a model. Obviously, our semi-supervised learning method has higher AUC than SVOREX, which demonstrate that incorporating unlabeled instances can improve the performance. 

Figure \ref{fig:time_vs_unlabeled} presents the training time against different size of unlabeled samples. The two lines of M-SAMULT and M-PNU-AUC are incomplete. This is because these two methods need to save the whole kernel matrix, and as unlabeled data continues to increase, they are all out of memory. In contrast, QS$^3$ORAO only need to keep $m$ random features in each iteration. This low memory requirement allows it to do an efficient training for large scale datasets. Besides, we can easily find that our method is faster than M-SAMULT and M-PNU-AUC when the number of unlabeled instances is larger than $10000$. This is because the M-SAMULT and M-PNU-AUC need $O(n^3)$ operations to compute the inverse of kernel matrix. Differently, QS$^3$ORAO uses RFF to approximate the kernel function, and each time it only needs $O(Dn)$ operations to calculate the random features with seed $i$. 

Based on these results, we conclude that QSG-S2AUC is superior to other state-of-the-art algorithms in terms of efficiency and scalability, while retaining similar generalization performance.


\section{Conclusion}
In this paper, we propose an unbiased objective function of semi-supervised OR AUC optimization and propose a novel scalable algorithm, QS$^3$ORAO to solve it. We decompose the original problem by $k-1$ parallel hyperplanes to $k-1$ binary semi-supervised AUC optimization problems. Then we use a DSG-based method  to achieve the optimal solution. Even though this optimization process contains four sources of randomness, theoretically, we prove that QS$^3$ORAO has a convergence rate of $O(1/t)$. The experimental results on various benchmark datasets also demonstrate the superiority of the proposed QS$^3$ORAO.

\section{ Acknowledgments}
This work was supported by Six talent peaks project (No. XYDXX-042) and the 333 Project (No. BRA2017455) in Jiangsu Province  and the National Natural Science Foundation of China (No: 61573191).
\bibliographystyle{aaai}
\bibliography{QSORAO}

\begin{thebibliography}{}

\bibitem[\protect\citeauthoryear{Belkin, Niyogi, and
  Sindhwani}{2006}]{belkin2006manifold}
Belkin, M.; Niyogi, P.; and Sindhwani, V.
\newblock 2006.
\newblock Manifold regularization: A geometric framework for learning from
  labeled and unlabeled examples.
\newblock {\em Journal of machine learning research} 7(Nov):2399--2434.

\bibitem[\protect\citeauthoryear{Chapelle, Scholkopf, and
  Zien}{2009}]{chapelle2009semi}
Chapelle, O.; Scholkopf, B.; and Zien, A.
\newblock 2009.
\newblock Semi-supervised learning.
\newblock {\em IEEE Transactions on Neural Networks} 20(3):542--542.

\bibitem[\protect\citeauthoryear{Chu and Keerthi}{2007}]{chu2007support}
Chu, W., and Keerthi, S.~S.
\newblock 2007.
\newblock Support vector ordinal regression.
\newblock {\em Neural computation} 19(3):792--815.

\bibitem[\protect\citeauthoryear{Dai \bgroup et al\mbox.\egroup
  }{2014}]{dai2014scalable}
Dai, B.; Xie, B.; He, N.; Liang, Y.; Raj, A.; Balcan, M.-F.~F.; and Song, L.
\newblock 2014.
\newblock Scalable kernel methods via doubly stochastic gradients.
\newblock In {\em Advances in NIPS},  3041--3049.

\bibitem[\protect\citeauthoryear{Drineas and
  Mahoney}{2005}]{drineas2005nystrom}
Drineas, P., and Mahoney, M.~W.
\newblock 2005.
\newblock On the nystr{\"o}m method for approximating a gram matrix for
  improved kernel-based learning.
\newblock {\em journal of machine learning research} 6(Dec):2153--2175.

\bibitem[\protect\citeauthoryear{Fathony, Bashiri, and
  Ziebart}{2017}]{fathony2017adversarial}
Fathony, R.; Bashiri, M.~A.; and Ziebart, B.
\newblock 2017.
\newblock Adversarial surrogate losses for ordinal regression.
\newblock In {\em Advances in NIPS},  563--573.

\bibitem[\protect\citeauthoryear{Fine and Scheinberg}{2001}]{fine2001efficient}
Fine, S., and Scheinberg, K.
\newblock 2001.
\newblock Efficient svm training using low-rank kernel representations.
\newblock {\em Journal of Machine Learning Research} 2(Dec):243--264.

\bibitem[\protect\citeauthoryear{Fujino and Ueda}{2016}]{fujino2016semi}
Fujino, A., and Ueda, N.
\newblock 2016.
\newblock A semi-supervised auc optimization method with generative models.
\newblock In {\em ICDM},  883--888.

\bibitem[\protect\citeauthoryear{Fullerton and
  Xu}{2012}]{fullerton2012proportional}
Fullerton, A.~S., and Xu, J.
\newblock 2012.
\newblock The proportional odds with partial proportionality constraints model
  for ordinal response variables.
\newblock {\em Social science research} 41(1):182--198.

\bibitem[\protect\citeauthoryear{F{\"u}rnkranz, H{\"u}llermeier, and
  Vanderlooy}{2009}]{furnkranz2009binary}
F{\"u}rnkranz, J.; H{\"u}llermeier, E.; and Vanderlooy, S.
\newblock 2009.
\newblock Binary decomposition methods for multipartite ranking.
\newblock In {\em Joint European Conference on Machine Learning and Knowledge
  Discovery in Databases},  359--374.
\newblock Springer.

\bibitem[\protect\citeauthoryear{Gao and Zhou}{2015}]{gao2015consistency}
Gao, W., and Zhou, Z.-H.
\newblock 2015.
\newblock On the consistency of auc pairwise optimization.
\newblock In {\em IJCAI},  939--945.

\bibitem[\protect\citeauthoryear{Gao \bgroup et al\mbox.\egroup
  }{2013}]{gao2013one}
Gao, W.; Jin, R.; Zhu, S.; and Zhou, Z.-H.
\newblock 2013.
\newblock One-pass auc optimization.
\newblock In {\em International Conference on Machine Learning},  906--914.

\bibitem[\protect\citeauthoryear{Geng \bgroup et al\mbox.\egroup
  }{2019}]{Geng2019ScalableSS}
Geng, X.; Gu, B.; Li, X.; Shi, W.; Zheng, G.; and Huang, H.
\newblock 2019.
\newblock Scalable semi-supervised svm via triply stochastic gradients.
\newblock In {\em IJCAI}.

\bibitem[\protect\citeauthoryear{Gu \bgroup et al\mbox.\egroup
  }{2015}]{Gu2015IncrementalSV}
Gu, B.; Sheng, V.~S.; Tay, K.; Romano, W.; and Li, S.
\newblock 2015.
\newblock Incremental support vector learning for ordinal regression.
\newblock {\em IEEE Transactions on Neural Networks and Learning Systems}
  26:1403--1416.

\bibitem[\protect\citeauthoryear{Gu \bgroup et al\mbox.\egroup
  }{2018a}]{gu2018accelerated}
Gu, B.; Shan, Y.; Geng, X.; and Zheng, G.
\newblock 2018a.
\newblock Accelerated asynchronous greedy coordinate descent algorithm for
  svms.
\newblock In {\em IJCAI},  2170--2176.

\bibitem[\protect\citeauthoryear{Gu \bgroup et al\mbox.\egroup
  }{2018b}]{gu2018asynchronous}
Gu, B.; Xin, M.; Huo, Z.; and Huang, H.
\newblock 2018b.
\newblock Asynchronous doubly stochastic sparse kernel learning.
\newblock In {\em Thirty-Second AAAI Conference on Artificial Intelligence}.

\bibitem[\protect\citeauthoryear{Gu, Huo, and Huang}{2019}]{gu2019scalable}
Gu, B.; Huo, Z.; and Huang, H.
\newblock 2019.
\newblock Scalable and efficient pairwise learning to achieve statistical
  accuracy.
\newblock In {\em Proceedings of the AAAI Conference on Artificial
  Intelligence}, volume~33,  3697--3704.

\bibitem[\protect\citeauthoryear{Gu, Xian, and
  Huang}{2019}]{gu2019asynchronous}
Gu, B.; Xian, W.; and Huang, H.
\newblock 2019.
\newblock Asynchronous stochastic frank-wolfe algorithms for non-convex
  optimization.
\newblock {\em IJCAI-19}.

\bibitem[\protect\citeauthoryear{Gu}{2018}]{gu2018regularization}
Gu, B.
\newblock 2018.
\newblock A regularization path algorithm for support vector ordinal
  regression.
\newblock {\em Neural Networks} 98:114--121.

\bibitem[\protect\citeauthoryear{Han \bgroup et al\mbox.\egroup
  }{2018}]{han2018co}
Han, B.; Yao, Q.; Yu, X.; Niu, G.; Xu, M.; Hu, W.; Tsang, I.; and Sugiyama, M.
\newblock 2018.
\newblock Co-teaching: Robust training of deep neural networks with extremely
  noisy labels.
\newblock In {\em Advances in NIPS},  8527--8537.

\bibitem[\protect\citeauthoryear{Kim and Ahn}{2012}]{kim2012corporate}
Kim, K.-j., and Ahn, H.
\newblock 2012.
\newblock A corporate credit rating model using multi-class support vector
  machines with an ordinal pairwise partitioning approach.
\newblock {\em Computers \& Operations Research} 39(8):1800--1811.

\bibitem[\protect\citeauthoryear{Li \bgroup et al\mbox.\egroup
  }{2017}]{Li2017TriplySG}
Li, X.; Gu, B.; Ao, S.; Wang, H.; and Ling, C.~X.
\newblock 2017.
\newblock Triply stochastic gradients on multiple kernel learning.
\newblock In {\em UAI}.

\bibitem[\protect\citeauthoryear{Liu \bgroup et al\mbox.\egroup
  }{2011}]{liu2011semi}
Liu, Y.; Liu, Y.; Zhong, S.; and Chan, K.~C.
\newblock 2011.
\newblock Semi-supervised manifold ordinal regression for image ranking.
\newblock In {\em Proceedings of the 19th ACM international conference on
  Multimedia},  1393--1396.
\newblock ACM.

\bibitem[\protect\citeauthoryear{Niu \bgroup et al\mbox.\egroup
  }{2016}]{niu2016ordinal}
Niu, Z.; Zhou, M.; Wang, L.; Gao, X.; and Hua, G.
\newblock 2016.
\newblock Ordinal regression with multiple output cnn for age estimation.
\newblock In {\em Proceedings of the CVPR},  4920--4928.

\bibitem[\protect\citeauthoryear{Rahimi and Recht}{2008}]{rahimi2008random}
Rahimi, A., and Recht, B.
\newblock 2008.
\newblock Random features for large-scale kernel machines.
\newblock In {\em Advances in NIPS},  1177--1184.

\bibitem[\protect\citeauthoryear{Rudin}{2017}]{rudin2017fourier}
Rudin, W.
\newblock 2017.
\newblock {\em Fourier analysis on groups}.
\newblock Courier Dover Publications.

\bibitem[\protect\citeauthoryear{Sakai, Niu, and
  Sugiyama}{2018}]{sakai2018semi}
Sakai, T.; Niu, G.; and Sugiyama, M.
\newblock 2018.
\newblock Semi-supervised auc optimization based on positive-unlabeled
  learning.
\newblock {\em Machine Learning} 107(4):767--794.

\bibitem[\protect\citeauthoryear{Seah, Tsang, and
  Ong}{2012}]{seah2012transductive}
Seah, C.-W.; Tsang, I.~W.; and Ong, Y.-S.
\newblock 2012.
\newblock Transductive ordinal regression.
\newblock {\em IEEE transactions on neural networks and learning systems}
  23(7):1074--1086.

\bibitem[\protect\citeauthoryear{Shi \bgroup et al\mbox.\egroup
  }{2019}]{Shi2019QuadruplySG}
Shi, W.; Gu, B.; Li, X.; Geng, X.; and Huang, H.
\newblock 2019.
\newblock Quadruply stochastic gradients for large scale nonlinear
  semi-supervised auc optimization.
\newblock In {\em IJCAI}.

\bibitem[\protect\citeauthoryear{Smola and
  Sch{\"o}lkopf}{2000}]{smola2000sparse}
Smola, A.~J., and Sch{\"o}lkopf, B.
\newblock 2000.
\newblock Sparse greedy matrix approximation for machine learning.

\bibitem[\protect\citeauthoryear{Srijith, Shevade, and
  Sundararajan}{2013}]{srijith2013semi}
Srijith, P.; Shevade, S.; and Sundararajan, S.
\newblock 2013.
\newblock Semi-supervised gaussian process ordinal regression.
\newblock In {\em Joint European conference on machine learning and knowledge
  discovery in databases},  144--159.
\newblock Springer.

\bibitem[\protect\citeauthoryear{Tsuchiya \bgroup et al\mbox.\egroup
  }{2019}]{DBLP:journals/corr/abs-1901-11351}
Tsuchiya, T.; Charoenphakdee, N.; Sato, I.; and Sugiyama, M.
\newblock 2019.
\newblock Semi-supervised ordinal regression based on empirical risk
  minimization.
\newblock {\em CoRR} abs/1901.11351.

\bibitem[\protect\citeauthoryear{Uematsu and
  Lee}{2014}]{uematsu2014statistical}
Uematsu, K., and Lee, Y.
\newblock 2014.
\newblock Statistical optimality in multipartite ranking and ordinal
  regression.
\newblock {\em IEEE transactions on pattern analysis and machine intelligence}
  37(5):1080--1094.

\bibitem[\protect\citeauthoryear{Waegeman and
  De~Baets}{2010}]{waegeman2010survey}
Waegeman, W., and De~Baets, B.
\newblock 2010.
\newblock A survey on roc-based ordinal regression.
\newblock In {\em Preference learning}. Springer.
\newblock  127--154.

\bibitem[\protect\citeauthoryear{Waegeman, De~Baets, and
  Boullart}{2008}]{waegeman2008roc}
Waegeman, W.; De~Baets, B.; and Boullart, L.
\newblock 2008.
\newblock Roc analysis in ordinal regression learning.
\newblock {\em Pattern Recognition Letters} 29(1):1--9.

\bibitem[\protect\citeauthoryear{Wang \bgroup et al\mbox.\egroup
  }{2015}]{wang2015optimizing}
Wang, S.; Li, D.; Petrick, N.; Sahiner, B.; Linguraru, M.~G.; and Summers,
  R.~M.
\newblock 2015.
\newblock Optimizing area under the roc curve using semi-supervised learning.
\newblock {\em Pattern recognition} 48(1):276--287.

\bibitem[\protect\citeauthoryear{Xie and Li}{2018}]{xie2018semi}
Xie, Z., and Li, M.
\newblock 2018.
\newblock Semi-supervised auc optimization without guessing labels of unlabeled
  data.

\bibitem[\protect\citeauthoryear{Yan}{2014}]{yan2014cost}
Yan, H.
\newblock 2014.
\newblock Cost-sensitive ordinal regression for fully automatic facial beauty
  assessment.
\newblock {\em Neurocomputing} 129:334--342.

\bibitem[\protect\citeauthoryear{Yu \bgroup et al\mbox.\egroup
  }{2019}]{yu2019tackle}
Yu, S.; Gu, B.; Ning, K.; Chen, H.; Pei, J.; and Huang, H.
\newblock 2019.
\newblock Tackle balancing constraint for incremental semi-supervised support
  vector learning.
\newblock In {\em Proceedings of the 25th ACM SIGKDD International Conference
  on Knowledge Discovery \& Data Mining},  1587--1595.
\newblock ACM.

\bibitem[\protect\citeauthoryear{Zhang \bgroup et al\mbox.\egroup
  }{2019}]{ijcai2019-590}
Zhang, C.; Ren, D.; Liu, T.; Yang, J.; and Gong, C.
\newblock 2019.
\newblock Positive and unlabeled learning with label disambiguation.
\newblock In {\em Proceedings of IJCAI-19},  4250--4256.

\end{thebibliography}

\appendix
\section{Detailed Proof of Ordered Rule of Thresholds}
\begin{proof}
	Firstly, take $i$-th subproblem into consideration, define the threshold is $b$. We have datasets $\{x_i\in\mathcal{D}_1\cup\cdots\cup\mathcal{D}_j|f(x_i)-b<0\}$ and $\{x_i\in\mathcal{D}_{j+1}\cup\cdots\cup\mathcal{D}_k|f(x_i)-b>0\}$.  Equation (15) is equivalent to 
	\begin{align}\label{theshold_rule}
	e_j(b) &= \sum_{x_i\in\mathcal{D}_1\cup\cdots\cup\mathcal{D}_j}(b-f(x_i))^2\notag\\&+\sum_{x_i\in\mathcal{D}_{j+1}\cup\cdots\cup\mathcal{D}_k}(f(x_i)-b)^2	\tag*{A}
	\end{align}
	According to the strict convexity of (\ref{theshold_rule}), that $b^*$ is unique. Then we can conclude that $\textbf{b}^*$ is unique. 
	
	Then we prove the order between thesholds.
	The derivative of $e_j(b)$ with respect to $b_k$ is 
	\begin{align}
	g_j(b) &= \dfrac{\partial e_j(b)}{\partial b}\nonumber\\
	& = 2\sum_{x_i\in\mathcal{D}_1\cup\cdots\cup\mathcal{D}_j}(b-f(x_i))\nonumber\\&-2\sum_{x_i\in\mathcal{D}_{j+1}\cup\cdots\cup\mathcal{D}_k}(f(x_i)-b)
	\end{align}
	Suppose $b_j^*\geq b_{j+1}^*$. We have $g_j(b_{j+1}^*)<0$ since $b_j^*$ is the minimizer of $e_j(b)$. Besides, since $b_{j+1}^*$ is minimizer of $e_{j+1}(b)$, we have $g_{k+1}(b_{j+1}^*)=0$. Thus, we have $g_{j+1}(b_{j+1}^*)-g_k(b_{j+1}^*)>0$. However, we also have 
	\begin{align}
	&\quad g_{j+1}(b_{j+1}^*)-g_k(b_{j+1}^*)\nonumber\\&=-\sum_{x_i\in\mathcal{D}_{j+1}}(b_{j+1}-f(x_i))-2\sum_{x_i\in\mathcal{D}_{j+1}}(f(x_i)-b_{j+1})\leq 0.
	\end{align}
	So we have $b_j^*<b_{j+1}^*$ and similary we can get $b_1^*<\cdots<b_{k-1}^*$
\end{proof}
\section{Detailed Proof of Convergence Rate}
In this section, we give detailed proof of Lemmas \ref{lemma:error_of_random_feature}-\ref{lemma:error_due_to_random_data} and Theorem \ref{theorem:Convergence_in_expectation}.

\subsection{Proof of Lemma \ref{lemma:error_of_random_feature}}
Here we give the detailed proof of lemma \ref{lemma:error_of_random_feature}.
\begin{proof}
	We denote $A_i(x)=A_i(x;x_i^p,x_i^n,x_i^u,\omega_i) := a_t^i(\zeta_i(x)-\xi_i(x))$. According to the assumption in section 5, $A_i(X)$ have a bound:	
	\begin{eqnarray}
	|A_i(x)| &\leq& |a_t^i|\left(|\zeta_i(x)|+|\xi_i(x)|\right)  \nonumber\\
	&\leq&|a_t^i|(\dfrac{1}{k-1}\sum_{j=1}^{k-1}(\gamma^j(l_1^jk(x_i^p,x)+l_2^jk(x_i^n,x))\nonumber\\&&+(1-\gamma^j)(l_3^jk(x_i^p,x)+l_4^jk(x_i^u,x)\nonumber\\&&+l_5^jk(x_i^u,x)+l_6^jk(x_i^n,x)))\nonumber \\&&+ \dfrac{1}{k-1}\sum_{j=1}^{k-1}(\gamma^j(l_1^j\phi_{\omega}(x^p)\phi_{\omega}(x)+l_2^j\phi_{\omega}(x^n)\phi_{\omega}(x)\nonumber\\&&+(1-\gamma^j)(l_3^j\phi_{\omega}(x^p)\phi_{\omega}(x)+l_4^j\phi_{\omega}(x^u)\phi_{\omega}(x)\nonumber\\&&+l_5^j\phi_{\omega}(x^u)\phi_{\omega}(x)+l_6^j\phi_{\omega}(x^n)\phi_{\omega}(x))))) \nonumber
	\end{eqnarray}
	\begin{eqnarray}	
	&\leq&|a_t^i|(\dfrac{1}{k-1}\sum_{j=1}^{k-1}(\gamma^j(M_1+M_2)\nonumber\\&&+(1-\gamma^j)(M_1+M_2+M_1+M_2))\kappa\nonumber\\&&+\dfrac{1}{k-1}(\gamma^j(M_1+M_2)\nonumber\\&&+(1-\gamma^j)(M_1+M_2+M_1+M_2))\phi)\nonumber\\
	&=& M(\kappa+\phi)|a_t^i | \nonumber
	\end{eqnarray}
	
	Then we obtain the lemma \ref{lemma:error_of_random_feature}. This completes the proof.
\end{proof}

\subsection{Proof of lemma \ref{lemma:upper_bound}}
\begin{lemma}\label{lemma:upper_bound}		
	Suppose $\eta_i = \dfrac{\theta}{i}$ ($1 \leq  i \leq t$) and $\theta\lambda \in (1,2)\cup\mathbb{Z}_{+}$. We have $|a_t^i| \leq \dfrac{\theta}{t}$ and $\sum_{i=1}^t |a_t^i|^2 \leq \dfrac{\theta^2}{t}$.
\end{lemma}
\begin{proof}
	Obviously, $|a_t^i| \leq \dfrac{\theta}{t}$. Then we have
	\begin{align}
	|a_t^i| &= |a_t^{i+1}\dfrac{\eta_i}{\eta_{i+1}}(1-\lambda\eta_{i+1})|\nonumber\\
	&= \dfrac{i+1}{i}|1-\dfrac{\lambda\theta}{i+1}|\cdot|a_t^{i+1}|\nonumber\\
	&= |\dfrac{i+1-\lambda\theta}{i}|\cdot|a_t^{i+1}| \nonumber
	\end{align}
	When $\lambda\theta \in (1,2), \forall i \geq 1$, we have $i-1<i+1-\lambda\theta < i$, so $|a_t^{i}|<|a_t^{i+1}|\leq \dfrac{\theta}{t}$ and $\sum_{i=1}^{t} \leq \dfrac{\theta^2}{t}$. When $\lambda\theta \in \mathbb{Z}_{+}$, if $i> \lambda\theta-1$, then $|a_t^i|< |a_t^{i+1}|\leq \dfrac{\theta}{t}$. If $i \leq \lambda\theta -1$, then $|a_t^i|=0$. So we get $\sum_{i=1}^t|a_t^i|^2 \leq \dfrac{\theta^2}{t}$.
	Therefore, we obtain the lemma \ref{lemma:upper_bound}. This completes the proof.
\end{proof}

\subsection{Proof of Lemma \ref{lemma:error_due_to_random_data}}
In $j$-th subproblem, we denote that $l_1'(h(x_t^p),h(x_t^n))$, $l_2'(h(x_t^p),h(x_t^n))$, $l_1'(h(x_t^p),h(x_t^u))$, $l_2'(h(x_t^p),h(x_t^u))$, $l_1'(h(x_t^u),h(x_t^n))$ and $l_2'(h(x_t^u),h(x_t^n))$ as $\tilde{l}_1^j$, $\tilde{l}_2^j$, $\tilde{l}_3^j$, $\tilde{l}_4^j$, $\tilde{l}_5^j$ and $\tilde{l}_6^j$, respectively. In addition, 

We define the following three different gradient terms,
\begin{eqnarray}
g_t &=& \xi_t + \lambda h_t\nonumber\\ 
&=&\dfrac{1}{k-1}\sum_{j=1}^{k-1}(\gamma^j(l_1^jk(x_t^p,\cdot)+l_2^jk(x_t^n,\cdot))\nonumber\\&&+(1-\gamma^j)(l_3^jk(x_t^p,\cdot)+l_4^jk(x_t^u,\cdot)\nonumber\\&&+l_5^jk(x_t^u,\cdot)+l_6^jk(x_t^n,\cdot)))+\lambda h_t \nonumber 
\end{eqnarray}
\begin{eqnarray}
\hat{g_t} &=&\hat{\xi_t} + \lambda h_t \nonumber\\
&=&\dfrac{1}{k-1}\sum_{j=1}^{k-1}(\gamma^j(\tilde{l}_1^jk(x_t^p,\cdot)+\tilde{l}_2^jk(x_t^n,\cdot))\nonumber\\&&+(1-\gamma^j)(\tilde{l}_3^jk(x_t^p,\cdot)+\tilde{l}_4^jk(x_t^u,\cdot)\nonumber\\&&+\tilde{l}_5^jk(x_t^u,\cdot)+\tilde{l}_6^jk(x_t^n,\cdot)))+\lambda h_t \nonumber 
\end{eqnarray}
\begin{eqnarray}
\bar{g_t}&=& \mathbb{E}_{x_t^p,x_t^n,x_t^u}[\hat{g_t}]\nonumber\\
&=&\mathbb{E}_{x_t^p,x_t^nx_t^u}[\dfrac{1}{k-1}\sum_{j=1}^{k-1}(\gamma(\tilde{l}_1^jk(x_t^p,\cdot)+\tilde{l}_2^jk(x_t^n,\cdot))\nonumber\\&&+(1-\gamma)(\tilde{l}_3^jk(x_t^p,\cdot)+\tilde{l}_4^jk(x_t^u,\cdot)\nonumber\\&&+\tilde{l}_5^jk(x_t^u,\cdot)+\tilde{l}_6^jk(x_t^n,\cdot)))]+\lambda h_t \nonumber
\end{eqnarray}

Note that from our previous definition, we have $h_{t+1}=h_t-\eta_t g_t, \forall t\geq 1$.

Denote $A_t=\parallel h_t - f^{*}\parallel_{\mathcal{H}}^2$. Then we have
\begin{eqnarray}
A_{t+1} &=& \parallel h_t - f^{*} - \eta_t g_t \parallel_{\mathcal{H}}^2 \nonumber \\
&=& A_t+ \eta_t^2 \parallel g_t \parallel_{\mathcal{H}}^2 - 2\eta_t \langle h_t - f^{*},g_t \rangle_{\mathcal{H}} \nonumber \\
&=& A_t + \eta_t^2 \parallel g_t \parallel_{\mathcal{H}}^2-2\eta_t \langle h_t - f^{*}, \bar{g_t} \rangle_{\mathcal{H}} \nonumber \\&& + 2\eta_t\langle h_t-f^{*}, \bar{g_t}-\hat{g_t} \rangle_{\mathcal{H}} + 2\eta_t \langle h_t-f^{*}, \hat{g_t}-g_t \rangle_{\mathcal{H}} \nonumber
\end{eqnarray}
Because of the strongly convexity of loss function and optimality condition, we have
\begin{eqnarray}
\langle h_t - f^{*},\bar{g_t} \rangle_{\mathcal{H}} \geq \lambda \parallel h_t - f^{*} \parallel_{\mathcal{H}}^2 \nonumber
\end{eqnarray}

Hence, we have
\begin{align}\label{random_data_err_function}
A_{t+1} &\leq (1-2\eta_t\lambda)A_t + \eta_t^2 \parallel g_t \parallel_{\mathcal{H}}^2 + 2\eta_t\langle h_t-f^{*},\bar{g_t}- \hat{g_t}  \rangle_{\mathcal{H}} \nonumber\\
& \quad + 2\eta_t \langle h_t-f^{*},\hat{g_t}-g_t \rangle_{\mathcal{H}}, \forall t \geq 1
\end{align}
Let us denote $\mathcal{M}_t = \parallel g_t \parallel_{\mathcal{H}}^2$, $\mathcal{N}_t = \langle h_t-f^{*},\bar{g_t}-\hat{g_t} \rangle_{\mathcal{H}}  $, $\mathcal{R}_t = \langle h_t-f^{*},\hat{g_t}-g_t\rangle_{\mathcal{H}} $. Firstly, we show that $\mathcal{M}_t$, $\mathcal{N}_t$, $\mathcal{R}_t$ are bounded. Specifically, for $t\geq 1$, we have
\begin{align}\label{Eq_M}
\mathcal{M}_t \leq \kappa M^2 (1+\lambda c_t)
\end{align}
where $c_t := \sqrt{\sum_{i,j=1}^{t-1}|a_{t-1}^{i}||a_{t-1}^{j}|}$
\begin{align}\label{Eq_N}
\mathbb{E}_{x_t^p,x_t^n,x_t^u,\omega_t}[\mathcal{N}_t] = 0
\end{align}
\begin{align}\label{Eq_R}
\mathbb{E}_{x_t^p,x_t^nx_t^u,\omega_t}[\mathcal{R}_t] \leq \kappa^{1/2}LB_{1,t}\sqrt{\mathbb{E}_{x_{t-1}^p,x_{t-1}^n,x_{t-1}^u,\omega_{t-1}}[A_t]}
\end{align}
where $M:=\dfrac{1}{k-1}\sum_{j=1}^{k-1}(2-\gamma^j)(M_1+M_2)$, $L:=\dfrac{1}{k-1}\sum_{j=1}^{k-1}(2-\gamma^j)(L_1+L_2)$ and $A_t = \parallel h_t - f^{*}\parallel_{\mathcal{H}}^2$.

\begin{proof}
	The proof of Eq. (\ref{Eq_M}):
	
	\begin{align}
	\mathcal{M}_t = \parallel g_t \parallel_{\mathcal{H}}^2 = \parallel \xi_t+\lambda h_t \parallel_{\mathcal{H}} ^2 \leq \left(\parallel \xi \parallel_{\mathcal{H}} + \lambda \parallel h_t \parallel_{\mathcal{H}}\right)^2 \nonumber
	\end{align}
	and
	\begin{eqnarray}
	\parallel \xi_t \parallel_{\mathcal{H}} &=& \parallel\dfrac{1}{k-1}\sum_{j=1}^{k-1} (\gamma^j (l_1^jk(x_t^p,\cdot)+l_2^jk(x_t^n,\cdot))\nonumber\\&&+(1-\gamma^j)(l_3^jk(x_t^p,\cdot) +l_4^jk(x_t^u,\cdot)\nonumber\\&&+ l_5^jk(x_t^u,\cdot)+l_6^jk(x_t^n,\cdot))) \parallel_{\mathcal{H}} \nonumber \\
	&\leq& \kappa^{1/2}(\dfrac{1}{k-1}\sum_{j=1}^{k-1}(\gamma^j(M_1+M_2)\nonumber\\&&+(1-\gamma^j)(M_1+M_2+M_1+M_2))) \nonumber\\
	&=&\kappa^{1/2}M \nonumber
	\end{eqnarray}
	Then we have:
	\begin{eqnarray}
	\|h_t\|_{\mathcal{H}}^2&=&\sum_{i=1}^{t-1}\sum_{j=1}^{t-1}a_{t-1}^{i}a_{t-1}^{j}\dfrac{1}{k-1}\sum_{j=1}^{k-1}[\gamma^j(l_1^jk(x_i^p,\cdot)+l_2^jk(x_i^n,\cdot))\nonumber\\&&+(1-\gamma^j)(l_3^jk(x_i^p,\cdot)+l_4^jk(x_i^u,\cdot)\nonumber\\&&+l_5^jk(x_i^u,\cdot)+l_6^jk(x_i^n,\cdot))]\nonumber\\&&\cdot\dfrac{1}{k-1}\sum_{j=1}^{k-1}[\gamma^j(l_1^jk(x_j^p,\cdot)+l_2^jk(x_j^n,\cdot))\nonumber\\&&+(1-\gamma^j)(l_3^jk(x_j^p,\cdot)+l_4^jk(x_j^u,\cdot)\nonumber\\&&+l_5^jk(x_j^u,\cdot)+l_6^jk(x_j^n,\cdot))]\nonumber\\
	&\leq&\kappa\dfrac{1}{(k-1)^2}\sum_{n,m}^{k-1}\sum_{i=1}^{t-1}\sum_{j=1}^{t-1}a_{t-1}^{i}a_{t-1}^{j}[\gamma^j\gamma^j(M_1+M_2)^2\nonumber\\&&+2\gamma^j(1-\gamma^j)(M_1+M_2)(M_1+M_2+M_1+M_2)\nonumber\\&&+(1-\gamma^j)^2(M_1+M_2+M_1+M_2)^2]\nonumber \\
	&=&\kappa\sum_{i=1}^{t-1}\sum_{j=1}^{t-1}a_{t-1}^{i}a_{t-1}^{j}[\dfrac{1}{k-1}\sum_{j=1}^{k-1}(2-\gamma^j)(M_1+M_2)]^2\nonumber\\
	&=& \kappa M^2\sum_{i=1}^{t-1}\sum_{j=1}^{t-1}a_{t-1}^{i}a_{t-1}^{j}  \nonumber
	\end{eqnarray}
	Then we obtain Eq. (\ref{Eq_M}). 
\end{proof}
\begin{proof}
	The proof of Eq. (\ref{Eq_N})
	\begin{eqnarray}
	&&\mathbb{E}_{x_t^p,x_t^n,x_t^u,\omega_t}[\mathcal{N}_t]\nonumber\\ 
	&=&\mathbb{E}_{x_{t-1}^p,x_{t-1}^n,x_{t-1}^u,\omega_t}\nonumber\\&&\left[\mathbb{E}_{x_{t}^p,x_t^n,x_{t}^u}[\langle h_t-f^{*},\bar{g_t}-\hat{g_t}\rangle_{\mathcal{H}}|x_{t-1}^p,x_{t-1}^u,\omega_t]\right] \nonumber\\
	&=& \mathbb{E}_{x_{t-1}^p,x_{t-1}^n,x_{t-1}^u,\omega_t}\left[\langle h_t-f^{*}, \mathbb{E}_{x_{t-1}^p,x_{t-1}^n,x_{t-1}^u}[\bar{g_t}-\hat{g_t}]\rangle_{\mathcal{H}}\right]\nonumber\\
	&=& 0 \nonumber
	\end{eqnarray}
\end{proof}

\begin{proof}
	The proof of Eq. (\ref{Eq_R})
	\begin{eqnarray}
	&& \mathbb{E}_{x_t^p,x_t^n,x_t^u,\omega_t}[\mathcal{R}_t]\nonumber\\	
	&=&\mathbb{E}_{x_{t}^p,x_t^n,x_{t}^u,\omega_t}[\langle h_t-f^{*},\hat{g_t}-g_t \rangle_{\mathcal{H}}] \nonumber\\	
	&=& \mathbb{E}_{x_{t}^p,x_t^n,x_{t}^u,\omega_t}[\langle h_t-f^{*}, \dfrac{1}{k-1}\sum_{j=1}^{k-1}\gamma^j[(l_1^j-\tilde{l}_1^j)k(x_t^p,\cdot)\nonumber\\&&+(l_2^j-\tilde{l}_2^j)k(x_t^n,\cdot)]\nonumber\\&&+(1-\gamma^j)[(l_3^j-\tilde{l}_3^j)k(x_t^p,\cdot)+(l_4^j-\tilde{l}_4^j)k(x_t^u,\cdot)\nonumber\\&&+(l_5^j-\tilde{l}_5^j)k(x_t^u,\cdot)+(l_6^j-\tilde{l}_6^j)k(x_t^n,\cdot)]  \rangle_{\mathcal{H}}] \nonumber
	\end{eqnarray}	
	\begin{eqnarray}	
	&\leq &\mathbb{E}_{x_{t}^p,x_t^p,x_{t}^u,\omega_t} [\parallel h_t-f^{*} \parallel_{\mathcal{H}}\nonumber\\&&\cdot \dfrac{1}{k-1}\sum_{j=1}^{k-1}\gamma^j[|l_1^j-\tilde{l}_1^j|\cdot\parallel k(x_t^p,\cdot)\parallel_{\mathcal{H}}\nonumber\\&&+|l_2^j-\tilde{l}_2^j|\parallel k(x_t^n,\cdot)\parallel_{\mathcal{H}}]\nonumber\\&&+(1-\gamma^j)[|l_3^j-\tilde{l}_3^j|\parallel k(x_t^p,\cdot)\parallel_{\mathcal{H}}+|l_4^j-\tilde{l}_4^j|\parallel k(x_t^u,\cdot)\parallel_{\mathcal{H}}\nonumber\\&&+|l_5^j-\tilde{l}_5^j|\parallel k(x_t^u,\cdot)\parallel_{\mathcal{H}}+|l_6^j-\tilde{l}_6^j|\parallel k(x_t^n,\cdot)\parallel_{\mathcal{H}}]   ] \nonumber\\
	&\leq& \kappa^{1/2}\mathbb{E}_{x_{t}^p,x_t^u,x_{t}^u,\omega_t} [\parallel h_t - f^{*} \parallel_{\mathcal{H}} \nonumber \\&&\cdot \dfrac{1}{k-1}\sum_{j=1}^{k-1}[\gamma^j(|l_1^j-\tilde{l}_1^j| +|l_2^j-\tilde{l}_2^j|)\nonumber\\&&+(1-\gamma^j)(|l_3^j-\tilde{l}_3^j| +|l_4^j-\tilde{l}_4^j|+|l_5^j-\tilde{l}_5^j| +|l_6^j-\tilde{l}_6^j|)]]\nonumber\\ 
	&\leq& \kappa^{1/2} \mathbb{E}_{x_{t}^p,x_t^n,x_{t}^u,\omega_t}[ \| h_t - f^{*} \|_{\mathcal{H}}\nonumber \\&&\cdot\dfrac{1}{k-1}\sum_{j=1}^{k-1}[\gamma^j(L_1|f_t(x_t^p)-h_t(x_t^p)|\nonumber\\&&+L_2|f_t(x_t^n)-h_t(x_t^n)|)\nonumber \\&&+(1-\gamma^j)(L_1|f_t(x_t^p)-h_t(x_t^p)|+L_2|f_t(x_t^u)-h_t(x_t^u)|\nonumber \\&&+L_1|f_t(x_t^u)-h_t(x_t^u)|+L_2|f_t(x_t^n)-h_t(x_t^n)|)] ] \nonumber \\
	&\leq& \kappa^{1/2}\sqrt{\mathbb{E}_{x_{t}^p,x_t^n,x_{t}^u,\omega_t}[A_t]} \nonumber\\&& \cdot \sqrt{[\dfrac{1}{k-1}\sum_{j=1}^{k-1}(2-\gamma^j)(L_1+L_2)]^2B_{1,t}^2}\nonumber\\
	&\leq& \kappa^{1/2}LB_{1,t}\sqrt{\mathbb{E}_{x_{t-1}^p,x_{t-1}^n,x_{t-1}^u,\omega_{t-1}}[A_t]} \nonumber
	\end{eqnarray}	
\end{proof}
The first and third inequalities are due to Cauchy-Schwarz Inequality and the second inequality is due to the Assumption 3. And the last step is due to the Lemma \ref{lemma:error_of_random_feature} and the definition of $A_t$.

After proving the Eq. (\ref{Eq_M})-(\ref{Eq_R}) separately, let us denote $e_t= \mathbb{E}_{x_{t}^p,x_{t}^u,\omega_{t}}[A_t]$, given the above bounds, we arrive the following recursion,
\begin{align}
e_{t+1} \leq (1-2\eta_t\lambda)e_t + \kappa M^2 \eta_t^2(1+\lambda c_t)^2+ 2\kappa^{1/2}L\eta_tB_{1,t}\sqrt{e_t}\nonumber
\end{align}

When $\eta_t = \dfrac{\theta}{t}$ with $\theta$ such that $\theta\lambda \in (1,2) \cup \mathbb{Z}_{+}$, from Lemma \ref{lemma:upper_bound}, we have $|a_t^i| \leq \dfrac{\theta}{t}$, $\forall 1 \leq i \leq t$. Consequently, $c_t \leq \theta$ and $B_{1,t}^2 \leq M^2(\kappa+\phi)^2\dfrac{\theta^2}{t-1}$. Applying these bounds to the above recursion, we have
$e_{t+1} \leq (1-\dfrac{2\theta\lambda}{t})e_t + \kappa^2\dfrac{\theta^2}{t^2}(1+\lambda\theta)^2 + 2\kappa^{1/2}\dfrac{\theta}{t}L\eta_t\sqrt{M^2(\kappa+\phi)^2 \dfrac{\theta^2}{t-1}}\sqrt{e_t}$ 
and then it can be rewritten as:
\begin{align}
e_{t+1} \leq (1-\dfrac{2\theta\lambda}{t})e_t+\dfrac{\beta_1}{t}\sqrt{\dfrac{e_t}{t}}+\dfrac{\beta_2}{t^2} \nonumber
\end{align}
where $\beta_1 = 2\sqrt{2}\kappa^{1/2}LM(\kappa+\phi)\theta^2$ ($\sqrt{\dfrac{t}{t-q}}\leq \sqrt{2}$) and $\beta_2 = \kappa M^2\theta^2(1+\lambda\theta)^2$. Reusing Lemma 14 in [Dai \textit{et al.}, 2014] with $\eta = 2\theta\lambda > 1$ leads to
\begin{align}
e_{t+1} \leq \dfrac{Q_1^2}{t} \nonumber
\end{align}
where
\begin{align}
Q_1 = \max\left\{\parallel f^{*} \parallel_{\mathcal{H}},\dfrac{Q_0 + \sqrt{Q_0^2+(2\theta\lambda-1)(1+\theta\lambda)^2\theta^2\kappa M^2}}{2\theta\lambda-1} \right\} \nonumber
\end{align}
and $Q_0 = 2\sqrt{2}\kappa^{1/2}(\kappa+\phi)LM\theta^2$.

\subsection{Proof of Theorem \ref{theorem:Convergence_in_expectation}}
%
\begin{proof}
	Substitute Lemma \ref{lemma:error_of_random_feature} and \ref{lemma:error_due_to_random_data} into Eq. (16), we have that
	\begin{eqnarray}
	\mathbb{E}_{x_{t}^p,x_{t}^u,\omega_t}[|f_t(x)-f^{*}(x)|^2]&\leq&2\mathbb{E}_{x_{t}^p,x_{t}^u,\omega_t}[|f_t(x)-h_{t+1}(x)|^2] \nonumber\\&& + 2\kappa\mathbb{E}_{x_{t}^p,x_{t}^u,\omega_t}[\parallel h_t - f^{*} \parallel_{\mathcal{H}}] \nonumber\\
	&\leq& 2B_{1,t+1}^2+ 2\kappa\dfrac{Q_1^2}{t} \nonumber \\
	&\leq& \dfrac{2C^2+2\kappa Q_1^2}{t} \nonumber
	\end{eqnarray}
	where the last inequality is due to Lemma \ref{lemma:upper_bound}. In this way, we obtain the final result on convergence in expectation. This completes the proof.
\end{proof}

\end{document}